\PassOptionsToPackage{hyphens}{url}
\documentclass{article}
\usepackage{preprint,times}
\usepackage[T1]{fontenc}
\usepackage[utf8]{inputenc}
\usepackage{latexsym}
\usepackage{microtype}
\usepackage{graphicx}
\usepackage{wrapfig}
\usepackage{amsmath}
\usepackage{amssymb}
\usepackage{mathtools}
\usepackage{amsthm}
\usepackage{booktabs}
\usepackage{multirow}
\usepackage{array}
\usepackage{placeins}
\usepackage{algorithm}
\usepackage{algorithmic}
\usepackage{xcolor}
\usepackage{caption}
\usepackage{authblk}
\newsavebox{\ArcherTableBox}
\newenvironment{ArcherFitTable}{\begin{lrbox}{\ArcherTableBox}}{%
  \end{lrbox}%
  \ifdim\wd\ArcherTableBox>\linewidth
    \resizebox{\linewidth}{!}{\usebox{\ArcherTableBox}}%
  \else
    \usebox{\ArcherTableBox}%
  \fi
}
\usepackage[colorlinks=true,linkcolor=blue,citecolor=blue,urlcolor=blue]{hyperref}
\usepackage{url}
\definecolor{tablegray}{gray}{0.40}
\newtheorem{proposition}{Proposition}

\newtheorem{corollary}{Corollary}
\newtheorem{lemma}{Lemma}
\newcommand{\Cprompt}{\mathcal C_p}
\newcommand{\Decoder}{\mathcal D}
\newcommand{\Mask}{\mathtt{MASK}}

\newcommand{\ArcherAppendixTableStyle}{%
  \small
  \renewcommand{\arraystretch}{1.08}%
  \setlength{\tabcolsep}{2.5pt}%
}
\renewcommand{\floatpagefraction}{0.75}

\iclrfinalcopy
\title{Archer: Adaptive Reuse of Cached Hidden States for Efficient Rollback in Diffusion Language Models}
\renewcommand{\Authfont}{\normalfont\bfseries}
\renewcommand{\Affilfont}{\normalfont}
\renewcommand{\Authsep}{, }
\renewcommand{\Authand}{, }
\renewcommand{\Authands}{, }
\author[1,2]{Xuning He\thanks{\href{mailto:hxning@mail.nankai.edu.cn}{\texttt{hxning@mail.nankai.edu.cn}}}}
\author[3]{Qingyuan Shi}
\author[4]{Zinan Sheng}
\author[4]{Yongding Tao}
\author[4]{Huanyu Liu}
\author[4]{Ge Li}
\author[4]{Xue Jiang}
\author[1]{Yihong Dong\thanks{\href{mailto:dongyh@sjtu.edu.cn}{\texttt{dongyh@sjtu.edu.cn}}}}
\affil[1]{School of Computer Science, Shanghai Jiao Tong University}
\affil[2]{College of Artificial Intelligence, Nankai University}
\affil[3]{School of Computer Science and Technology, Huazhong University of Science and Technology}
\affil[4]{School of Computer Science, Peking University}
\newcommand{\ArcherEmailText}{\texttt{hxning@mail.nankai.edu.cn} \quad \texttt{dongyh@sjtu.edu.cn}}
\hypersetup{
 pdftitle={Archer: Adaptive Reuse of Cached Hidden States for Efficient Rollback in Diffusion Language Models},
 pdfauthor={Xuning He; Qingyuan Shi; Zinan Sheng; Yongding Tao; Huanyu Liu; Ge Li; Xue Jiang; Yihong Dong},
 pdfsubject={Archer single-column preprint},
 pdfkeywords={diffusion language models, rollback, KV caching}
}

\begin{document}
\maketitle
\begin{abstract}
\frenchspacing
\urlstyle{same}
Diffusion language models (DLMs) can revise earlier predictions as context evolves, but rollback requires costly full-sequence recomputation. Conventional key-value (KV) caching
assumes immutable historical states and is difficult to reconcile with such revision. We
introduce Adaptive Reuse of Cached Hidden States for Efficient Rollback (Archer), a
training-free caching method for rollback-capable DLMs. Archer recomputes the mutable response
at every update while reusing prompt K/V within a bounded response-state neighborhood. This
asymmetric boundary amortizes prompt computation without restricting rollback. Bounded reuse
also delays feedback from tentative tokens while keeping response revision active.
Our analysis characterizes the cache boundary, bounds
state-dependent approximation error, and gives a decoder-margin condition for preserving
full-refresh decisions. On code-generation benchmarks, Archer achieves the best mean
performance of $33.63\%$ with a $2.57\times$ mean speedup, reaching $2.95\times$ speedup on a
single benchmark. On non-code benchmarks, Archer achieves up to $4.33\times$ speedup over
LLaDA-Confidence while improving average accuracy. Controlled analyses support the delayed
prompt feedback mechanism.
Our code is available at \url{https://github.com/Hxnng/Archer}.
\end{abstract}

\section{Introduction}
\label{sec:introduction}

Diffusion language models (DLMs) generate through iterative, bidirectional denoising rather
than an irreversible left-to-right factorization
\citep{austin2021structured,li2022diffusionlm,sahoo2024simple}. This evolving state enables
rollback: earlier predictions can be reconsidered when later context reveals inconsistencies.
Recent rollback-capable decoders demonstrate its value for structured generation
\citep{wang2025remdm,hong2025wino,dong2026saber}. Yet rollback requires many full-sequence
forward passes, with costs growing as prompts and outputs lengthen. Once rollback is
used at scale, KV caching becomes necessary; the
central question is how to introduce it without weakening the ability to revise the
generation.

Rollback breaks the premise behind conventional KV caching. In autoregressive decoding,
causal attention fixes previous contexts, allowing K/V reuse
\citep{pope2023efficiently,kwon2023pagedattention}. Under bidirectional attention, changing a
generated token alters generation and prompt states. Full recomputation preserves rollback
but reprocesses an unchanged prompt; full-sequence caching saves this work but retains
activations from superseded generations. Eager prompt updates feed
tentative tokens into later predictions, potentially reinforcing transient errors before
subsequent context can correct them. The challenge is therefore to locate a cache boundary
that saves computation while leaving every revisable
generation state current.

\begin{figure}[t]
    \centering
    \includegraphics[width=\textwidth]{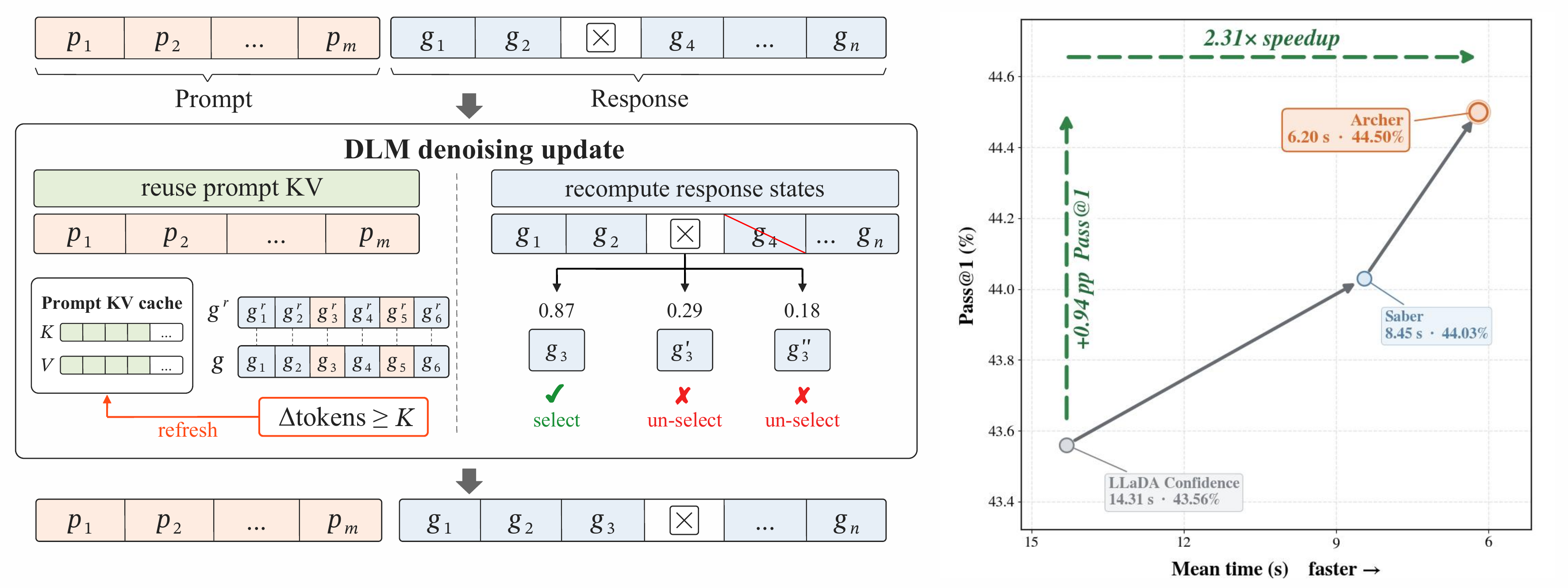}
    \caption{Archer reuses prompt K/V while recomputing the revisable response, reducing
    latency and improving Pass@1 on MBPP.}
    \label{fig:overview}
\end{figure}

Existing methods improve rollback quality or DLM efficiency without resolving their
interaction. Token-level rollback decoders improve quality but typically repeat full-sequence
computation. DLM accelerators reduce
computation through blockwise generation, parallel token acceptance, delayed KV updates, or
selective recomputation
\citep{arriola2025block,wu2025fastdllm,ma2025dkvcache,liu2025dllmcache}, but target decoding
without token-level rollback. Their generation caches must be frequently invalidated or retain
states derived from revised tokens, so efficient reuse and unrestricted rollback remain at
odds. A method for scalable rollback must make this trade-off explicit rather than assume
that conventional cache reuse transfers unchanged.

In this paper, we introduce \textbf{A}daptive \textbf{R}euse of \textbf{C}ached \textbf{H}idden States for
\textbf{E}fficient \textbf{R}ollback (\textbf{Archer}) to make this interaction explicit.
Archer caches only prompt K/V states and recomputes the complete generation at every step,
preserving rollback for every generated token. It refreshes the prompt cache when the current
generation has moved sufficiently far from the state at which the cache was created. This
state-aware policy amortizes prompt computation while preventing unbounded staleness. Bounded
reuse also forms a temporal anchor that delays feedback from tentative tokens and can reduce
premature error reinforcement. Our analysis identifies prompt states as a cache boundary
aligned with rollback, derives the computational gain and a state-dependent approximation
bound, and gives a decoder-margin condition under which Archer preserves the rollback decision
of full recomputation. Archer thus turns the tension between caching and rollback into a
controlled lag-and-reset process that alternates local reuse with global synchronization.

Across the main benchmarks, Archer breaks the usual quality--speed trade-off, achieving the
best average performance at $33.63\%$, a $2.57\times$ average speedup, and up to $2.95\times$
speedup on a single benchmark. Across DLM backbones, it improves
Pass@1 by up to $3.05$ percentage points over Saber. On five non-code benchmarks, Archer
achieves $48.37\%$ macro-average accuracy with a $2.47\times$ mean speedup over
LLaDA-Confidence. Controlled analyses support delayed prompt
feedback and favor response-state distance over cache age for identifying cache-induced
decision changes. We \textbf{1)} formulate the conflict between KV
caching and rollback and identify prompt states as the appropriate boundary; \textbf{2)}
propose Archer with state-aware refresh and characterize its efficiency, approximation error,
and decision fidelity; and \textbf{3)} show that rollback-compatible caching moves the DLM
quality--speed frontier.

\section{Motivation}
\label{sec:motivation}

Rollback is a defining advantage of DLM decoding over an irreversible left-to-right process
\citep{wang2025remdm,hong2025wino,dong2026saber}. It allows the current output to remain
provisional until later context resolves earlier uncertainty, but every change is followed by
another bidirectional forward pass over an increasingly long sequence. As prompts and outputs
scale, repeatedly recomputing the entire context is no longer practical, so KV caching becomes
necessary for rollback to remain usable. The design problem is to obtain this reuse without
turning a revisable generation into fixed history. The resulting systems question is not
whether to cache, but where reuse can coexist with unrestricted revision inside the same
decoder.

\begin{figure}[t]
    \centering
    \includegraphics[width=\textwidth]{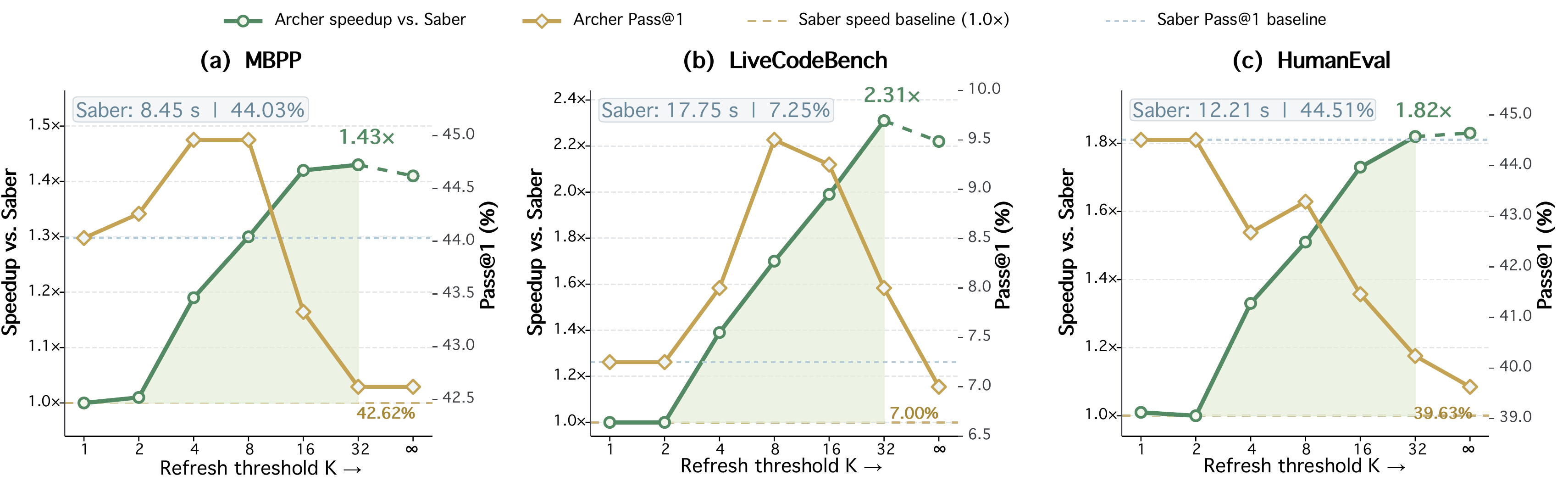}
    \caption{Effect of prompt-cache radius $K$. Longer reuse reduces latency, while quality
    peaks at moderate $K$ and declines with excessive staleness.}
    \label{fig:k-sensitivity}
\end{figure}

The cache boundary must follow the semantics of rollback rather than the convenience of
implementation. Generation states represent the part of the sequence that may
change, and the affected positions are not known in advance. At step $t$, the rollback set is
determined from the current context-dependent logits, which can be summarized as $\mathcal
R_t=\mathcal B(F_\theta(p,g_t))$. If the unreliable tokens were known before
evaluating $F_\theta$, the decoder could remove them directly; rollback is needed because
their reliability must first be reassessed as context evolves. Generation caching therefore
creates a circular dependency. Determining which cached K/V and logits remain valid requires
the same bidirectional generation computation that caching is intended to avoid, and one
revised token can invalidate every state that depends on it. Exact generation-side reuse
consequently approaches full recomputation, while approximate reuse introduces assumptions
that may restrict rollback. Preserving unrestricted rollback therefore keeps the mutable
generation outside the persistent cache and recomputes it before every decoding decision.

Prompt states provide the boundary that satisfies both requirements. Their values depend on
the current output, but their token identities remain fixed, so reuse introduces contextual
delay without making any generated token permanent. Archer therefore recomputes the complete
generation at every step and reuses only prompt K/V. The eager feedback loop is
\begin{equation}
g_t
\longrightarrow \mathcal C_p(g_t)
\longrightarrow z_{t+1}
\longrightarrow g_{t+1},
\label{eq:motivation-feedback}
\end{equation}
where $\mathcal C_p(g_t)$ is the prompt state induced by the current generation. Replacing it
temporarily with $\mathcal C_p(g_r)$ removes repeated prompt computation and delays the
feedback of tentative tokens. The former improves speed; the latter gives rollback more
opportunity to correct transient high-confidence errors before they reinforce themselves. This
asymmetry follows the semantics of revision rather than an arbitrary implementation choice.

The cache lifetime controls whether this boundary yields a useful trade-off.
Figure~\ref{fig:k-sensitivity} shows that increasing the reuse radius $K$ consistently reduces
latency because prompt refreshes become less frequent. Quality is non-monotonic. A very small
radius behaves similarly to eager execution and provides little reuse or temporal separation.
A moderate radius delays premature feedback while keeping the prompt sufficiently
representative of the current generation. An excessively large radius withholds useful context
for too long, allowing approximation error to dominate. The rise-and-fall in quality shows
that cache reuse must balance temporal anchoring against synchronization rather than maximize
either one in isolation.

Rollback and KV caching become compatible when reuse follows what the decoder is allowed to
change. Prompt-only reuse keeps the generation current, delays feedback for a bounded
interval, and restores the full context through synchronization. Archer operationalizes this
boundary as the state-anchored lag-and-reset policy studied below.

\begin{figure}[!t]
    \centering
    \includegraphics[width=\textwidth]{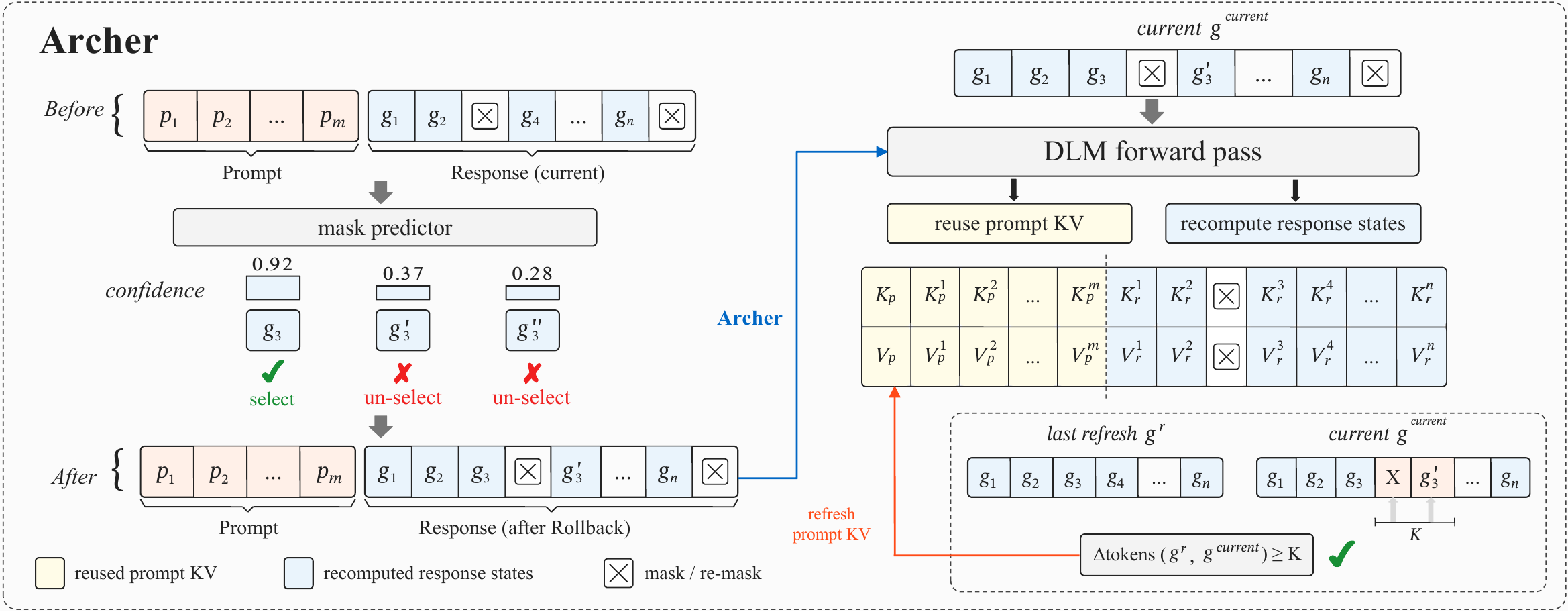}
    \caption{Archer's state-anchored prompt cache. Response states are recomputed at every
    update, and prompt K/V is refreshed when $d_H(g_t,g_r)\geq K$.}
    \label{fig:archer-overview}
\end{figure}

\section{Related Work}
\label{sec:related-work}

\subsection{Rollback in Diffusion Language Models}

DLMs replace left-to-right factorization with iterative denoising in continuous embeddings or
discrete token spaces \citep{li2022diffusionlm,austin2021structured}. Advances in discrete
objectives have substantially improved their modeling quality
\citep{lou2024sedd,sahoo2024simple,li2024promises}, and recent systems such as LLaDA, Dream,
and DiffuCoder have scaled the paradigm to general language modeling and code generation
\citep{nie2025llada,ye2025dream,gong2025diffucoder}. Their changing intermediate states are
not merely a sampling detail; they provide the rollback capability that motivates our systems
design.

Recent decoders exploit rollback through flexible sampling. ReMDM derives a
remasking transition, RemeDi learns to identify unreliable predictions, and WINO and Saber
revisit token-level decisions as context evolves
\citep{wang2025remdm,huang2025remedi,hong2025wino,dong2026saber}. These token-level methods
generally process the revised sequence with another full forward pass. Reversible Diffusion
Decoding instead returns to earlier blocks using cached block states \citep{wang2026rdd}, but
does not address repeated prompt computation during fine-grained revision inside a
bidirectional generation region. Archer treats that interaction as a systems constraint.
Because generated tokens may change, it reuses only fixed prompt states, leaving the
sampling policy and its correction mechanism unchanged.

\subsection{Efficient DLM Inference}

Efficient DLM inference reduces either denoising steps or their per-step cost. Fast-dLLM,
EB-Sampler, WINO, and Saber resolve multiple positions per pass according to confidence,
entropy, or state changes
\citep{wu2025fastdllm,benhamu2025ebsampler,hong2025wino,dong2026saber}. These methods shorten
the trajectory, and Fast-dLLM additionally uses blockwise KV caching. Without caching,
each remaining pass still processes the full bidirectionally coupled sequence as prompts
and outputs grow. The bottleneck therefore shifts from how many iterations
are executed to how much repeated context each surviving iteration processes.

KV caching is exact for the immutable history of causal decoding
\citep{pope2023efficiently,kwon2023pagedattention}, and fixed prompt modules can be shared
across requests \citep{gim2024promptcache}. For DLMs, Block Diffusion creates cacheable
semi-autoregressive structure during training \citep{arriola2025block}, while training-free
methods use prefix or dual caches, delayed token-level updates, and similarity-guided partial
recomputation \citep{wu2025fastdllm,ma2025dkvcache,liu2025dllmcache}. These approaches target
conventional denoising and do not treat rollback as a cache-invalidation event.
Introducing them into a rollback decoder can retain generation-side computation from
tokens later re-masked or replaced. Archer derives the cache boundary from rollback semantics.
It amortizes prompt computation while recomputing every revisable generation state, so reuse
neither commits a provisional token nor alters rollback semantics.

\section{Archer}
\label{sec:archer}

Archer reuses prompt K/V while recomputing the response at every step.
Response-state drift triggers cache refreshes (Figure~\ref{fig:archer-overview}).

\subsection{Asymmetric State Reuse}

Let $p\in\mathcal V^P$ be the prompt, $g_t\in(\mathcal V\cup\{\mathtt{MASK}\})^G$ the current
response, and $\mathcal D$ the rollback decoder. Any auxiliary sampler state is suppressed for
notation. We write $\mathcal C_p(g)$ for the prompt K/V obtained under response $g$ and
$\Phi_\theta(g;\mathcal C)$ for the generation forward using prompt cache $\mathcal C$. Eager
decoding evaluates
\begin{equation}
z_t^\star
=\Phi_\theta(g_t;\mathcal C_p(g_t)),
\qquad
g_{t+1}=\mathcal D(g_t,z_t^\star).
\label{eq:archer-fresh}
\end{equation}

Suppose the current cache was created at response $g_r$. Archer replaces the eager logits with
\begin{equation}
\hat z_t
=\Phi_\theta(g_t;\mathcal C_p(g_r)).
\label{eq:archer-cached}
\end{equation}
The generation path is otherwise fresh. At layer $\ell$, its queries attend to anchored prompt
K/V and current generation K/V,
\begin{equation}
\operatorname{Attn}\!\left(
Q_g^\ell(g_t),
\begin{bmatrix}
K_p^\ell(g_r)\\
K_g^\ell(g_t)
\end{bmatrix},
\begin{bmatrix}
V_p^\ell(g_r)\\
V_g^\ell(g_t)
\end{bmatrix}
\right).
\label{eq:archer-attention}
\end{equation}
Hence every generation embedding, hidden state, and logit reflects $g_t$; only
response-to-prompt feedback remains anchored at $g_r$.

\subsection{State-Anchored Refresh}

Archer measures cache validity by the response change accumulated since the most recent
refresh. The resulting anchor-relative distance is
\begin{equation}
D_t
=d_H(g_t,g_r)
=\sum_{i=1}^{G}\mathbb I[g_{t,i}\ne g_{r,i}]
\label{eq:archer-distance}
\end{equation}

Archer refreshes when $D_t\ge K$. A refresh performs a full forward, replaces the prompt
cache, and sets $g_r\leftarrow g_t$. Otherwise Archer evaluates Eq.~\eqref{eq:archer-cached}.
This makes synchronization depend on accumulated response drift rather than on the number of
updates alone. The logits passed to the original decoder are therefore
\begin{equation}
\tilde z_t=
\begin{cases}
\Phi_\theta(g_t;\mathcal C_p(g_t)), & D_t\ge K,\\
\Phi_\theta(g_t;\mathcal C_p(g_r)), & D_t<K,
\end{cases}
\qquad
g_{t+1}=\mathcal D(g_t,\tilde z_t).
\label{eq:archer-controller}
\end{equation}
The controller responds to net state drift rather than elapsed iterations. Several negligible
updates may continue to reuse an anchor, while one large rollback can trigger immediate
synchronization.

Algorithm~\ref{alg:archer} gives the complete procedure. \textsc{FullForward} returns exact
logits and a new prompt cache; \textsc{CachedForward} reconstructs the current generation
using that cache. Token acceptance, replacement, re-masking, and termination remain entirely
governed by $\mathcal D$.

\begin{algorithm}[!htbp]
\caption{State-anchored rollback decoding with Archer}
\label{alg:archer}
\begin{algorithmic}[1]
\REQUIRE Prompt $p$, initial response $g_0$, rollback decoder $\mathcal D$, radius $K$
\STATE $g\leftarrow g_0$; $g_r\leftarrow g_0$
\STATE $(z_{\mathrm{init}},\mathcal C_r)\leftarrow\textsc{FullForward}(p,g)$
\STATE $z\leftarrow\textsc{CachedForward}(g;\mathcal C_r)$
\WHILE{$\neg\textsc{Terminated}(g)$}
    \STATE $g\leftarrow\mathcal D(g,z)$
    \IF{$\neg\textsc{Terminated}(g)$}
        \IF{$d_H(g,g_r)\ge K$}
            \STATE $(z,\mathcal C_r)\leftarrow\textsc{FullForward}(p,g)$
            \STATE $g_r\leftarrow g$
        \ELSE
            \STATE $z\leftarrow\textsc{CachedForward}(g;\mathcal C_r)$
        \ENDIF
    \ENDIF
\ENDWHILE
\RETURN $g$
\end{algorithmic}
\end{algorithm}



\section{Theoretical Analysis}
\label{sec:theory}

Archer keeps response states current while allowing prompt feedback to lag within a bounded
response neighborhood. This asymmetry links computational reuse to revisability, approximation
error, and feedback dynamics. Proofs are in Appendix~\ref{app:theory}.

\subsection{Caching under Revision}

Under dense bidirectional attention, a response edit can affect every contextual state,
including positions whose token identities remain unchanged.

\begin{proposition}[Rollback-compatible cache boundary]
\label{prop:rollback-boundary}
A generic exact-reuse policy must account for the full dependency closure of a response edit.
Archer caches no generation state and preserves every replacement and re-masking action
available to the original decoder.
\end{proposition}

Every revised token is re-embedded and propagated through all layers before the next decision.
Prompt staleness may change the selected action, but cannot make a response position immutable.

\subsection{Efficiency and Fidelity}

Let $C_{\rm full}$ and $C_{\rm cache}$ denote full and cached forward costs. A cached step
computes $G$ response query paths over all $P+G$ positions, omitting prompt queries,
projections, and feed-forward work.

\begin{proposition}[Amortized computational gain]
\label{prop:computational-gain}
Under the cost model in Appendix~\ref{app:cost-proof}, ignoring cache overhead,
$C_{\rm cache}/C_{\rm full}=G/(P+G)$. For equal numbers of forwards and Archer full-forward
fraction $q_K$, including initialization,
\begin{equation}
S_{\rm end}(K)=\left[q_K+(1-q_K)\frac{G}{P+G}\right]^{-1}.
\label{eq:theory-speedup}
\end{equation}
\end{proposition}

The partition determines the saving per cached step, and refresh frequency determines its
realization. Prompt K/V nevertheless remain anchored at $g_r$ while the response evolves to
$g_t$.

Let $L_C$ be the local Lipschitz constant from response embeddings to prompt K/V, $L_\Phi$
from prompt K/V to logits, and $B$ the token embedding diameter, including the mask embedding.
Set $\Gamma=L_\Phi L_C B$. Let $m_{\mathcal D}(g,z)$ be the $\ell_\infty$ distance to a logit
perturbation that changes the next decoding transition, with auxiliary state suppressed.

\begin{proposition}[State-bounded error and decision fidelity]
\label{prop:state-fidelity}
Under these smoothness assumptions, at a common sampler state, every reused step satisfies
\begin{equation}
\|z_t^\star-\hat z_t\|_\infty\le\Gamma\sqrt{D_t}.
\label{eq:theory-radius-bound}
\end{equation}
The next transition agrees with eager decoding whenever
\begin{equation}
\Gamma\sqrt{D_t}<m_{\mathcal D}(g_t,z_t^\star).
\label{eq:theory-decision}
\end{equation}
From a common initialization, deterministic trajectories coincide if this condition holds at
every reused step.
\end{proposition}

Thus response drift, rather than cache age, controls the error bound. Reuse while $D_t<K$
bounds the perturbation; refresh at $D_t\ge K$ resets it.

\subsection{Prompt Reuse as Feedback Control}

Eager prompt updates incorporate provisional response tokens into prompt states and feed them
back into later predictions. Anchoring delays this return path until refresh, while response
states are recomputed from the current hypothesis at every update. Let $y=E(g)$ stack response
embeddings in a continuous relaxation, and write $J^{\rm fresh}$ and $J_r^{\rm anchor}$ for
the eager and anchored logit Jacobians.

\begin{proposition}[Selective delay of prompt feedback]
\label{prop:delayed-feedback}
For differentiable prompt-cache and response maps, within a fixed anchor interval,
\begin{equation}
\begin{aligned}
J^{\rm fresh}(y)
&=\partial_y\Phi_\theta(y;\mathcal C_p(y))
+\partial_{\mathcal C}\Phi_\theta(y;\mathcal C_p(y))
\,\partial_y\mathcal C_p(y),\\
J_r^{\rm anchor}(y)
&=\partial_y\Phi_\theta(y;\mathcal C_p(y_r)).
\end{aligned}
\label{eq:theory-jacobian}
\end{equation}
\end{proposition}

The direct term retains current response interaction, allowing later response context to inform
replacement and re-masking. The additional eager term describes response-to-prompt-to-response
feedback, which anchoring withholds while keeping rollback active. If this path reinforces a
provisional error, delaying it can leave room for correction; if it carries useful context, the
same delay can hurt. Refresh synchronizes prompt states with the current response and resets
the accumulated cache discrepancy. The radius $K$ therefore balances saved computation and
delayed feedback against the drift allowed before synchronization.

\section{Experimental Results}
\label{sec:experiments}

All shared experimental settings follow Saber \citep{dong2026saber}. We use MBPP, HumanEval,
and LiveCodeBench \citep{austin2021mbpp,chen2021humaneval,jain2025livecodebench}.
Pass@1 is the percentage
of tasks whose single generated program passes all tests; for MBPP and HumanEval, Base and ET
denote the original and Extended Test Cases suites, respectively \citep{dong2025codescore}.
We report mean generation time in seconds and speedup, using LLaDA-Confidence as the
reference. Appendices~\ref{app:k-sensitivity}--\ref{app:experimental-setup}
provide supplementary experiments, analyses, and evaluation configurations.

\subsection{Main Results}
\label{sec:main-results}

\begin{table}[!htbp]
\centering
\small
\setlength{\tabcolsep}{1pt}

\begin{ArcherFitTable}
\begin{tabular}{@{\extracolsep{\fill}}lccccccccccccc@{}}
\toprule
& \multicolumn{4}{c}{MBPP} & \multicolumn{3}{c}{LiveCodeBench} & \multicolumn{4}{c}{HumanEval} & \multicolumn{2}{c}{Overall} \\
\cmidrule(lr){2-5}\cmidrule(lr){6-8}\cmidrule(lr){9-12}\cmidrule(l){13-14}
\multirow{2}{*}[-1.5pt]{Method}
& \multicolumn{2}{c}{Pass@1 $\uparrow$}
& \multirow{2}{*}[-1.5pt]{Time $\downarrow$}
& \multirow{2}{*}[-1.5pt]{Spd. $\uparrow$}
& \multicolumn{1}{c}{Pass@1 $\uparrow$}
& \multirow{2}{*}[-1.5pt]{Time $\downarrow$}
& \multirow{2}{*}[-1.5pt]{Spd. $\uparrow$}
& \multicolumn{2}{c}{Pass@1 $\uparrow$}
& \multirow{2}{*}[-1.5pt]{Time $\downarrow$}
& \multirow{2}{*}[-1.5pt]{Spd. $\uparrow$}
& \multirow{2}{*}[-1.5pt]{Perf. $\uparrow$}
& \multirow{2}{*}[-1.5pt]{Spd. $\uparrow$} \\
\cmidrule(lr){2-3}\cmidrule(lr){6-6}\cmidrule(lr){9-10}
& Base & ET & & & Base & & & Base & ET & & & & \\
\midrule
\multicolumn{14}{l}{\textcolor{tablegray}{\textit{Standard DLM Sampling}}} \\
LLaDA-Confidence
& 43.56 & 30.44 & 14.31 & 1.00$\times$
& \underline{8.75} & 28.98 & 1.00$\times$
& \underline{43.29} & 35.37 & 19.86 & 1.00$\times$
& 32.28 & 1.00$\times$ \\
\midrule
\multicolumn{14}{l}{\textcolor{tablegray}{\textit{Efficient DLM Sampling}}} \\
Fast-dLLM + Cache
& \underline{44.50} & 31.62 & 11.85 & 1.21$\times$
& 7.75 & 12.75 & 2.27$\times$
& 41.46 & 36.59 & 11.95 & 1.66$\times$
& 32.38 & 1.71$\times$ \\
Fast-dLLM + Parallel
& 40.05 & 29.27 & \textbf{3.40} & \textbf{4.21}$\times$
& 8.25 & \textbf{4.85} & \textbf{5.98}$\times$
& 41.46 & 36.59 & \textbf{4.60} & \textbf{4.32}$\times$
& 31.12 & \textbf{4.84}$\times$ \\
dKV-Cache-Decode
& \textbf{45.67} & \underline{32.32} & 12.62 & 1.13$\times$
& \textbf{10.00} & 17.66 & 1.64$\times$
& 42.07 & 35.98 & 14.54 & 1.37$\times$
& \underline{33.21} & 1.38$\times$ \\
dKV-Cache-Greedy
& 23.19 & 18.27 & 12.36 & 1.16$\times$
& 2.00 & 21.25 & 1.36$\times$
& 15.24 & 14.63 & 15.82 & 1.26$\times$
& 14.67 & 1.26$\times$ \\
dLLM-Cache
& 44.03 & \underline{32.32} & 9.28 & 1.54$\times$
& \textbf{10.00} & 11.53 & 2.51$\times$
& 40.24 & 34.15 & 10.16 & 1.96$\times$
& 32.15 & 2.00$\times$ \\
\midrule
\multicolumn{14}{l}{\textcolor{tablegray}{\textit{Rollback-Capable DLM Decoding}}} \\
Saber
& 44.03 & 31.38 & 8.45 & 1.69$\times$
& 7.25 & 17.75 & 1.63$\times$
& \textbf{44.51} & \underline{37.20} & 12.21 & 1.63$\times$
& 32.87 & 1.65$\times$ \\
\textbf{Archer}
& \underline{44.50} & \textbf{32.55} & \underline{6.20} & \underline{2.31}$\times$
& \textbf{10.00} & \underline{9.81} & \underline{2.95}$\times$
& \underline{43.29} & \textbf{37.80} & \underline{8.07} & \underline{2.46}$\times$
& \textbf{33.63} & \underline{2.57}$\times$ \\
\bottomrule
\end{tabular}
\end{ArcherFitTable}
\caption{Comparison with DLM sampling and cache baselines; \textbf{bold} and
\underline{underlined} values mark the best and second-best results, respectively.}
\label{tab:main-results}
\end{table}

Table~\ref{tab:main-results} compares Archer with standard LLaDA decoding \citep{nie2025llada},
efficient methods designed for DLM decoding without rollback (Fast-dLLM, dKV-Cache, and
dLLM-Cache \citep{wu2025fastdllm,ma2025dkvcache,liu2025dllmcache}), and rollback-capable Saber.
All methods share the model revision, prompts, response budget, and evaluator.
Fast-dLLM with parallel decoding achieves $4.84\times$ mean speedup at $31.12\%$ average
performance; dKV-Cache-Decode offers higher quality but only $1.38\times$ speedup.

Archer reduces Saber's repeated prompt computation while preserving its rollback rules,
lowering latency on all three benchmarks and improving Pass@1 on both ET suites.
Its $33.63\%$ average performance is the highest, and its $2.57\times$ mean speedup over
LLaDA-Confidence makes Archer the only method in the top two for both metrics.

\subsection{Generalization Beyond Code}
\label{sec:noncode-results}

\begin{figure}[!htbp]
    \centering
    \includegraphics[width=\textwidth]{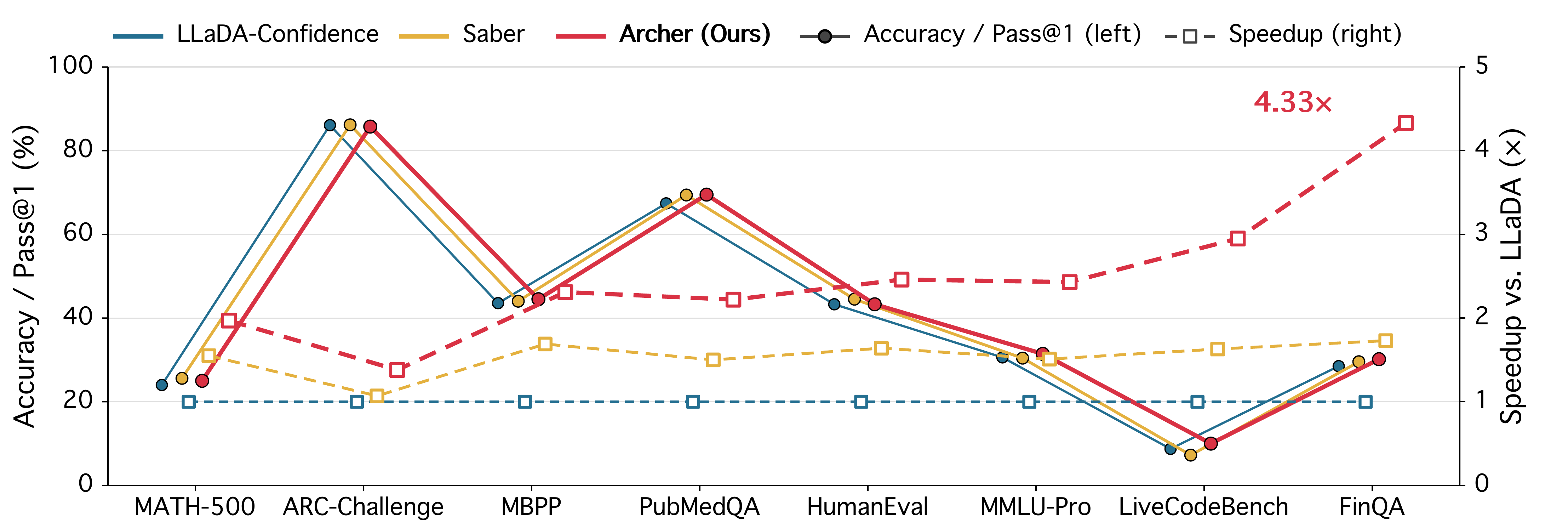}
    \caption{Accuracy/Pass@1 and speedup over LLaDA-Confidence across code
    and non-code benchmarks.}
    \label{fig:noncode-generalization}
\end{figure}

We evaluate LLaDA-8B-Instruct on MATH-500 \citep{hendrycks2021math,lightman2024verify}, FinQA
\citep{chen2021finqa}, PubMedQA \citep{jin2019pubmedqa}, ARC-Challenge \citep{clark2018arc},
and MMLU-Pro (Engineering and Physics) \citep{wang2024mmlupro} under matched prompts and
scoring; settings and results are in Appendices~\ref{app:generalization}
and~\ref{app:experimental-setup}.
Saber improves macro-average accuracy over LLaDA-Confidence at $1.47\times$ mean speedup
(Figure~\ref{fig:noncode-generalization}).
Archer reaches $4.33\times$ speedup on FinQA with a $1.66$-point accuracy gain against the
baseline.
Across the five tasks, Archer achieves the highest macro-average accuracy ($48.37\%$) with
$2.47\times$ mean speedup, supporting transfer across domains.

\subsection{Generalization across DLM Backbones}
\label{sec:backbone-results}

\begin{table}[!htbp]
\centering
\small
\setlength{\tabcolsep}{2pt}

\begin{ArcherFitTable}
\begin{tabular}{@{\extracolsep{\fill}}lcccccccccccccc@{}}
\toprule
& \multicolumn{4}{c}{LLaDA-8B}
& \multicolumn{4}{c}{Dream-7B}
& \multicolumn{4}{c}{DiffuCoder-7B}
& \multicolumn{2}{c}{Overall} \\
\cmidrule(lr){2-5}\cmidrule(lr){6-9}\cmidrule(lr){10-13}\cmidrule(l){14-15}
\multirow{2}{*}[-1.5pt]{Method}
& \multicolumn{2}{c}{Pass@1 $\uparrow$}
& \multirow{2}{*}[-1.5pt]{Time $\downarrow$}
& \multirow{2}{*}[-1.5pt]{Spd. $\uparrow$}
& \multicolumn{2}{c}{Pass@1 $\uparrow$}
& \multirow{2}{*}[-1.5pt]{Time $\downarrow$}
& \multirow{2}{*}[-1.5pt]{Spd. $\uparrow$}
& \multicolumn{2}{c}{Pass@1 $\uparrow$}
& \multirow{2}{*}[-1.5pt]{Time $\downarrow$}
& \multirow{2}{*}[-1.5pt]{Spd. $\uparrow$}
& \multirow{2}{*}[-1.5pt]{Perf. $\uparrow$}
& \multirow{2}{*}[-1.5pt]{Spd. $\uparrow$} \\
\cmidrule(lr){2-3}\cmidrule(lr){6-7}\cmidrule(lr){10-11}
& Base & ET & & & Base & ET & & & Base & ET & & & & \\
\midrule
\multicolumn{15}{l}{\textit{MBPP}} \\
\addlinespace[2pt]
Saber
& 44.03 & 31.38 & 8.45 & 1.00$\times$
& 52.93 & 41.69 & 8.18 & 1.00$\times$
& 59.48 & 44.73 & 12.35 & 1.00$\times$
& 45.71 & 1.00$\times$ \\
\textbf{Archer}
& \textbf{44.50} & \textbf{32.55} & \textbf{6.20} & \textbf{1.36}$\times$
& \textbf{53.63} & \textbf{42.15} & \textbf{5.51} & \textbf{1.48}$\times$
& \textbf{60.19} & \textbf{45.20} & \textbf{8.97} & \textbf{1.38}$\times$
& \textbf{46.37} & \textbf{1.41}$\times$ \\
\midrule
\multicolumn{15}{l}{\textit{HumanEval}} \\
\addlinespace[2pt]
Saber
& \textbf{44.51} & 37.20 & 12.21 & 1.00$\times$
& 29.88 & \textbf{28.05} & 10.40 & 1.00$\times$
& 57.32 & 51.83 & 15.57 & 1.00$\times$
& 41.46 & 1.00$\times$ \\
\textbf{Archer}
& 43.29 & \textbf{37.80} & \textbf{8.07} & \textbf{1.51}$\times$
& \textbf{31.10} & \textbf{28.05} & \textbf{5.85} & \textbf{1.78}$\times$
& \textbf{60.37} & \textbf{54.27} & \textbf{9.24} & \textbf{1.69}$\times$
& \textbf{42.48} & \textbf{1.66}$\times$ \\
\bottomrule
\end{tabular}
\end{ArcherFitTable}
\caption{Cross-backbone comparison with Saber on MBPP and HumanEval.}
\label{tab:backbone-results}
\end{table}

We evaluate training-free transfer on LLaDA-8B-Instruct \citep{nie2025llada},
Dream-v0-Instruct-7B \citep{ye2025dream}, and DiffuCoder-7B-cpGRPO \citep{gong2025diffucoder}
using MBPP and HumanEval.
Table~\ref{tab:backbone-results} shows $1.36$--$1.78\times$ speedup over Saber across all six
model--benchmark pairs, with higher average performance on both benchmarks and Pass@1 gains
up to $3.05$ percentage points.
These results support transfer across independently trained DLMs; model-specific decoding
settings are given in Appendix~\ref{app:experimental-setup}.

\subsection{Effect of Delayed Prompt Feedback}
\label{sec:feedback-probe}

\begin{wrapfigure}{r}{0.60\textwidth}
    \centering
    \small
    \captionsetup{font=small,skip=4pt}
    \includegraphics[width=\linewidth]{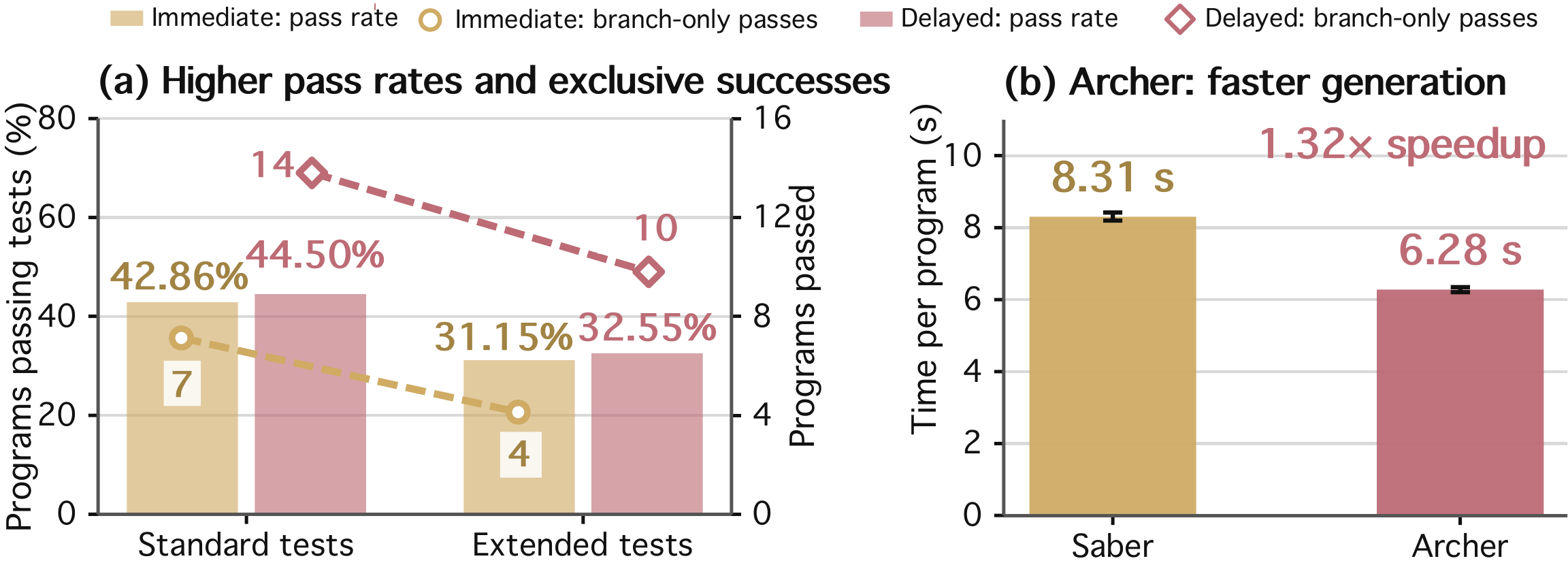}
    \caption{Delayed prompt feedback: paired correctness and decoding efficiency on MBPP.
    (a) Pass@1 (bars, left axis) and branch-exclusive successes (lines, right axis) for
    immediate and delayed refresh on Base and ET tests.
    (b) Independent end-to-end Saber--Archer generation time (mean $\pm$ SD).}
    \label{fig:feedback-statistics}
\end{wrapfigure}

We examine whether delayed prompt feedback can improve final correctness while keeping
response revision active. This tests the mechanism in Proposition~\ref{prop:delayed-feedback}:
synchronizing prompt states with tentative predictions may reinforce them before later
context enables correction.

On MBPP, we fork the same response and sampler state after a qualifying token is accepted,
while its prediction is still absent from the cached prompt snapshot.
\emph{Immediate refresh} rebuilds prompt K/V once before the next update;
\emph{delayed refresh} retains the preceding snapshot. Both branches then follow the same
normal refresh and rollback rules through completion. We compare final correctness within
each pair and track subsequent revision of the selected token. Independent generation
runs compare standard Archer with eager Saber for efficiency. Full protocols are given in
Appendix~\ref{app:mechanism}.

Figure~\ref{fig:feedback-statistics}(a) shows higher Pass@1 with delayed refresh on both
suites: gains of $1.64$ percentage points on Base and $1.40$ on ET. Delayed-only successes
are $2\times$ and $2.5\times$ as numerous as immediate-only successes on Base and ET,
respectively. These paired outcomes favor retaining the existing anchor over introducing an
early refresh. This quality trend does not require freezing accepted predictions: the delayed
branch re-masks $34.61\%$ of selected tokens before termination, showing that they remain open
to revision as the response evolves. The quality improvements remain suggestive, since paired
intervals include zero (Appendix~\ref{app:feedback}).

Separately, Figure~\ref{fig:feedback-statistics}(b) shows $24.45\%$ lower generation latency
for Archer ($1.32\times$ speedup over Saber), despite slightly more model calls. Cached prompt
forwards account for $81.62\%$ of these calls (Appendix~\ref{app:feedback}). Prompt reuse thus
lowers work per update while preserving response revision.

\subsection{Why Refresh by State Distance?}
\label{sec:validity-signal}

\begin{wrapfigure}{r}{0.60\textwidth}
    \centering
    \small
    \captionsetup{font=small,skip=4pt}
    \includegraphics[width=\linewidth]{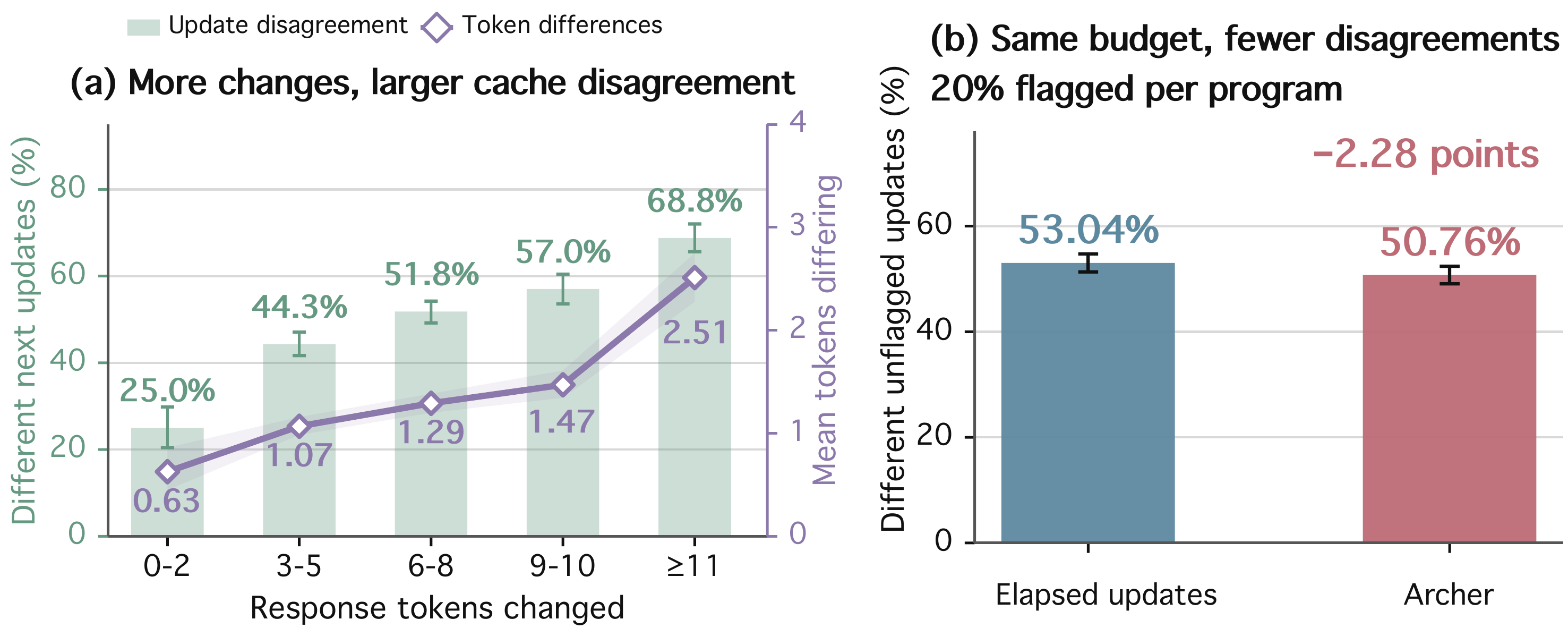}
    \caption{State distance better identifies cache-sensitive updates.
    (a) Next-update disagreement (bars, left axis) and mean differing response positions
    (line, right axis) against state distance, evaluated before the refresh decision.
    (b) Residual disagreement after offline rankings by cache age or Archer's state distance
    flag identical fractions of states per problem. Error bars and shading show $95\%$
    bootstrap intervals.}
    \label{fig:cache-validity-statistics}
\end{wrapfigure}

We test whether response-state distance is a more informative refresh signal than elapsed
updates. Proposition~\ref{prop:state-fidelity} links cache error to response drift, but the
operational question is whether cached prompt states change the decoder's next action. A
single rollback can revise many positions, whereas other updates leave the response almost
unchanged.

On MBPP, we compare old-cache reuse with a freshly rebuilt prompt cache at the same response
and sampler state, before the refresh decision. Both alternatives use the same response-forward
path, and these diagnostic updates leave the main trajectory unchanged. We measure whether the
next response or acceptance/re-masking decisions differ, and how many positions disagree between
the two next responses. To compare refresh signals fairly, we rank the same candidate states
offline by state distance $D_t=d_H(g_t,g_r)$ or cache age $A_t=t-r$, letting each score flag the
same fraction within each problem. Full protocols are in Appendix~\ref{app:cache-validity-protocol}.

Figure~\ref{fig:cache-validity-statistics}(a) shows a monotonic increase across distance
groups: next-update disagreement rises from $25.00\%$ to $68.80\%$, while the mean number of
differing response positions grows approximately fourfold, from $0.63$ to $2.51$. Cache-age
groups do not exhibit the same ordered trend. Distance also remains positively associated with
both outcomes after controlling for age (Appendix~\ref{app:cache-validity}), suggesting that
it captures information about cache sensitivity beyond elapsed updates.

Under the same selection budget, distance concentrates more cache-sensitive states in the
flagged set: their disagreement rate is $68.70\%$, compared with $59.59\%$ for age
(Appendix~\ref{app:cache-validity}). The remaining states consequently have fewer disagreeing
updates: $50.76\%$ versus $53.04\%$ in Figure~\ref{fig:cache-validity-statistics}(b), a
$2.28$-point reduction whose paired $95\%$ interval excludes zero. Their mean token discrepancy
is also $15.17\%$ lower. These offline comparisons support state distance as a stronger refresh
signal: it better targets cache-sensitive decisions under the same selection budget.

\section{Conclusion}

We introduced Archer, a training-free method that combines state-aware prompt reuse with
response recomputation to accelerate DLMs while preserving rollback for every token. Results
across tasks and backbones show improved quality--speed trade-offs, while controlled analyses
suggest that delayed prompt feedback can aid later error correction.

\label{page:main-end}

\FloatBarrier
\section*{Limitations}
\label{sec:limitations}

Our exploration of the refresh radius covers a limited set of configurations
rather than an exhaustive search over workload-specific operating points.
Within this scope, Archer already improves the average quality--latency
trade-off across the evaluated tasks and backbones. The reported settings
therefore leave open the possibility of further optimization. More extensive
configuration search, or adaptive radius selection, may uncover better
operating points for different workloads and latency budgets.

Archer introduces additional work for response-state distance checks and
prompt-cache management. Nevertheless, our end-to-end timings include cache
construction, refreshes, and decoder control, and show that reduced prompt
recomputation outweighs these costs in the evaluated settings. Our timing
evaluation currently focuses on single-example decoding; the balance between
saved computation and cache-management overhead may differ under batched
serving or other attention backends. Extending evaluation to these settings
would clarify how the measured latency savings translate into throughput and
identify opportunities for further implementation optimization.

\FloatBarrier
\bibliography{references}
\bibliographystyle{preprint}

\clearpage
\appendix
\setlength{\parskip}{5pt}
\raggedbottom
\setcounter{topnumber}{3}
\setcounter{dbltopnumber}{3}
\setcounter{totalnumber}{5}
\renewcommand{\topfraction}{0.95}
\renewcommand{\dbltopfraction}{0.85}
\renewcommand{\textfraction}{0.05}
\renewcommand{\floatpagefraction}{0.75}
\renewcommand{\dblfloatpagefraction}{0.70}
\setlength{\textfloatsep}{12pt plus 2pt minus 2pt}
\setlength{\dbltextfloatsep}{12pt plus 2pt minus 2pt}
\setlength{\floatsep}{10pt plus 2pt minus 2pt}
\setlength{\dblfloatsep}{10pt plus 2pt minus 2pt}
\captionsetup{justification=justified,singlelinecheck=false,skip=6pt}
\makeatletter
\setlength{\@fptop}{0pt}
\setlength{\@fpsep}{12pt}
\setlength{\@fpbot}{0pt plus 1fil}
\setlength{\@dblfptop}{0pt}
\setlength{\@dblfpsep}{12pt}
\setlength{\@dblfpbot}{0pt plus 1fil}
\makeatother
\section{Formal Analysis and Proofs}
\label{app:theory}

This section proves Propositions~\ref{prop:rollback-boundary}--\ref{prop:delayed-feedback}.
Appendix~\ref{app:boundary-proof} establishes the cache boundary;
Appendix~\ref{app:cost-proof} derives the computational gain;
Appendices~\ref{app:error-proof} and~\ref{app:decision-proof} prove the error and fidelity
claims; and Appendix~\ref{app:feedback-proof} formalizes delayed prompt feedback.
The analysis distinguishes revisability, numerical approximation, and trajectory fidelity.
Prompt-only caching preserves the ability to revise any response position, while its effect
on the selected transition depends on both cache error and the decoder's decision margin.

\subsection{Notation and Decoder State}

Let $p\in\mathcal V^P$ be a fixed prompt and let
$g_t\in(\mathcal V\cup\{\Mask\})^G$ be the response after update $t$.
The complete sampler state is denoted by
$\xi_t=(g_t,u_t)$, where $u_t$ collects confidence histories and any other
state used by the rollback rule. A full bidirectional forward can be written as
\begin{equation}
z_t^\star=\Phi_\theta(g_t;\Cprompt(g_t)),
\qquad
\xi_{t+1}=\Decoder(\xi_t,z_t^\star).
\label{eq:app-eager}
\end{equation}
Here $\Cprompt(g)$ stacks the prompt K/V states from every transformer layer
when the response is $g$. If the most recent cache was created at $g_r$,
Archer instead evaluates
\begin{equation}
\hat z_t=\Phi_\theta(g_t;\Cprompt(g_r)),
\qquad
\hat\xi_{t+1}=\Decoder(\xi_t,\hat z_t).
\label{eq:app-cached}
\end{equation}
Equations~\eqref{eq:app-eager} and~\eqref{eq:app-cached} compare the two
forwards at the \emph{same} sampler state. After their decisions differ, the two trajectories
need not remain at the same $g_t$. The model and prompt are fixed, and forwards are
deterministic with dropout disabled. Any additional model-conditioning variables must also
be identical in the compared forwards; dependence on them is suppressed in the notation.
We take $K\ge1$ to be an integer.

\subsection{Why Response-State Reuse Conflicts with Arbitrary Rollback}
\label{app:boundary-proof}

Let $\Delta^0$ be the response positions changed by a rollback, and define the structural
dependency closure at layer $\ell$ by
\begin{equation}
\Delta^\ell=\{i:\ \exists j\in\Delta^{\ell-1}
\text{ with an attention edge }j\!\rightarrow\! i\}.
\label{eq:app-closure}
\end{equation}

\begin{lemma}[Dense rollback closure]
\label{lem:app-closure}
For a transformer layer with dense bidirectional attention,
$\Delta^0\neq\varnothing$ implies $\Delta^1=\{1,\ldots,P+G\}$.
Every contextual response state therefore lies in the structural dependency closure of the
edit.
\end{lemma}

\begin{proof}
Every query position $i$ attends to every key position $j$. Choose any
$j\in\Delta^0$. The edge $j\!\rightarrow\! i$ exists for every $i$, so every
position belongs to $\Delta^1$. Later layers inherit this full structural closure. The claim
concerns possible dependence rather than the magnitude of a particular numerical change.
An unchanged token identity alone does not certify that its contextual state remains
identical after another response position changes.
\end{proof}

\begin{proof}[Proof of Proposition~\ref{prop:rollback-boundary}]
Lemma~\ref{lem:app-closure} places every contextual response state in the dependency closure
of an arbitrary edit. A generic exact-reuse policy must account for this closure rather than
rely on an accepted token's identity. Archer instead stores no response tensor across updates.
At a cached forward, response queries attend to cached prompt K/V and newly computed response
K/V. Response embeddings, queries, keys, values, hidden states, and logits are all recomputed
from the current response. Archer then passes the full response-logit tensor to the
unmodified decoder. Every replacement or re-masking action available to $\Decoder$ remains
representable. Approximate prompt K/V can change the selected action, but cannot remove a
response position from the decoder's revision domain.
\end{proof}

The structural argument is not a computational lower bound. It does not exclude reuse of
unchanged embeddings or their first-layer K/V projections, sparse dependency regions, or
model-specific exactness certificates. It excludes a universal guarantee for contextual
response states based only on unchanged token identities. Archer's rollback compatibility
means preservation of revisability; trajectory fidelity requires the additional margin
condition below.

\subsection{End-to-End Computational Cost}
\label{app:cost-proof}

\begin{proof}[Proof of Proposition~\ref{prop:computational-gain}]
Let $N=P+G$ and let $L$ be the number of transformer layers. At layer $\ell$, write
$\alpha_\ell$ for the position-wise projection and feed-forward cost and $\beta_\ell$ for
one query--key interaction. The leading-cost model is
\begin{equation}
C_{\rm full}=\sum_{\ell=1}^{L}
\left[\alpha_\ell N+\beta_\ell N^2\right].
\label{eq:app-full-cost}
\end{equation}
An Archer cached forward evaluates only $G$ fresh query paths, each still attending to all
$N$ positions. Including cache assembly and dispatch overhead $H$, its cost is
\begin{equation}
C_{\rm cache}=\sum_{\ell=1}^{L}
\left[\alpha_\ell G+\beta_\ell GN\right]+H.
\label{eq:app-cache-cost}
\end{equation}
The same coefficients apply to both paths, so
\begin{equation}
C_{\rm cache}=\frac{G}{N}C_{\rm full}+H,
\qquad
\eta=\frac{G}{N}+\frac{H}{C_{\rm full}}.
\label{eq:app-cost-ratio}
\end{equation}

Let eager decoding use $M_0$ forwards and Archer use $M_K$ forwards. If $R_K$ of Archer's
forwards are full evaluations, including initialization, the total-cost ratio is
\begin{equation}
S(K)=\frac{M_0 C_{\rm full}}
{R_K C_{\rm full}+(M_K-R_K)C_{\rm cache}}.
\label{eq:app-end-cost}
\end{equation}
With $q_K=R_K/M_K$, factoring numerator and denominator gives
\begin{equation}
S(K)=\frac{M_0}{M_K}\left[q_K+(1-q_K)\eta\right]^{-1},
\label{eq:app-speedup-exact}
\end{equation}
Setting $M_0=M_K$ and $H=0$ proves Eq.~\eqref{eq:theory-speedup}.
\end{proof}

For equal trajectory lengths and negligible overhead,
$S(K)=[q_K+(1-q_K)G/(P+G)]^{-1}$. The ideal cached-step speedup is $1+P/G$;
initialization and subsequent refreshes reduce its end-to-end realization. This is an
operation-count model, not a wall-clock guarantee. Kernel utilization, memory movement,
cache construction, and decoder control can change measured latency, while the factor
$M_0/M_K$ explicitly allows the number of forwards to change when decisions differ.

\subsection{From Prompt Staleness to Logit Error}
\label{app:error-proof}

We first prove the logit-error claim in Proposition~\ref{prop:state-fidelity}. Let $E(g)$
stack response token embeddings. The local smoothness assumptions in
Proposition~\ref{prop:state-fidelity} are
\begin{equation}
\begin{gathered}
\|\Cprompt(g)-\Cprompt(g')\|
\le L_C\|E(g)-E(g')\|_F,\\
\|\Phi_\theta(g;C)-\Phi_\theta(g;C')\|_\infty
\le L_\Phi\|C-C'\|.
\end{gathered}
\label{eq:app-cache-smoothness}
\end{equation}
The same cache norm is used in both inequalities. The constants must hold over the response
and cache pairs being compared within the anchor neighborhood. Let $B$ bound the diameter of
token embeddings, including $\Mask$. The neighborhood constant in
Proposition~\ref{prop:state-fidelity} is $\Gamma=L_\Phi L_C B$.

\begin{proof}[Proof of the error bound in Proposition~\ref{prop:state-fidelity}]
Only the $D_t$ positions at which $g_t$ and $g_r$ differ contribute to the embedding
discrepancy. Each contributes at most $B^2$ to its squared Frobenius norm, so
\begin{equation}
\|E(g_t)-E(g_r)\|_F^2\le B^2 D_t.
\label{eq:app-embedding-hamming}
\end{equation}
Applying the two Lipschitz inequalities in sequence yields
\begin{equation}
\begin{aligned}
\|z_t^\star-\hat z_t\|_\infty
&\le L_\Phi\|\Cprompt(g_t)-\Cprompt(g_r)\|\\
&\le L_\Phi L_C\|E(g_t)-E(g_r)\|_F
\le\Gamma\sqrt{D_t}.
\end{aligned}
\label{eq:app-radius-bound}
\end{equation}
A reused step satisfies $D_t<K$. Since both $D_t$ and $K$ are integers,
\begin{equation}
\|z_t^\star-\hat z_t\|_\infty\le\Gamma\sqrt{K-1}.
\label{eq:app-k-bound}
\end{equation}
A refresh evaluates the full model at $g_t$ and resets the anchor to $g_t$, so the
cache-induced error at that forward is zero.
\end{proof}

\paragraph{Layerwise refinement.}
The product bound can also be resolved into individual transformer layers. Let
$H_{g,t}^{\ell}$ and $\hat H_{g,t}^{\ell}$ be eager and cached response hidden states after
layer $\ell$, both evaluated at $g_t$, and define
\begin{equation}
e_\ell=\|H_{g,t}^{\ell}-\hat H_{g,t}^{\ell}\|,
\qquad e_0=0.
\label{eq:app-layer-error}
\end{equation}
At layer $\ell$, define prompt-cache staleness by
\begin{equation}
s_t^\ell=\|K_p^\ell(g_t)-K_p^\ell(g_r)\|
+\|V_p^\ell(g_t)-V_p^\ell(g_r)\|.
\label{eq:app-staleness}
\end{equation}
For $C=(K_p^\ell,V_p^\ell)$, use the sum norm in Eq.~\eqref{eq:app-staleness}. Assume
that the response-layer update satisfies, for constants $a_\ell,b_\ell\ge0$,
\begin{equation}
\|\Phi_\ell(H,C)-\Phi_\ell(H',C')\|
\le a_\ell\|H-H'\|+b_\ell\|C-C'\|.
\label{eq:app-layer-lipschitz}
\end{equation}

\begin{lemma}[Layerwise propagation]
\label{lem:app-propagation}
Under Eq.~\eqref{eq:app-layer-lipschitz},
\begin{equation}
e_\ell\le a_\ell e_{\ell-1}+b_\ell s_t^\ell.
\label{eq:app-recurrence}
\end{equation}
\end{lemma}

\begin{proof}
Add and subtract $\Phi_\ell(\hat H_{g,t}^{\ell-1},C_t^\ell)$ between the eager and cached
updates, where $C_t^\ell$ is the fresh prompt K/V. The triangle inequality separates the
error inherited from the previous response layer and the new error caused by prompt K/V.
Applying Eq.~\eqref{eq:app-layer-lipschitz} yields Eq.~\eqref{eq:app-recurrence}.
\end{proof}

Unrolling the recurrence and applying an $L_{\rm head}$-Lipschitz output head gives
\begin{equation}
\|z_t^\star-\hat z_t\|_\infty
\le L_{\rm head}\sum_{\ell=1}^{L}
b_\ell s_t^\ell\prod_{j=\ell+1}^{L}a_j.
\label{eq:app-unrolled}
\end{equation}
If the layer-$\ell$ prompt-cache map is $\kappa_\ell$-Lipschitz in $E(g)$ under the
corresponding sum norm, then $s_t^\ell\le\kappa_\ell B\sqrt{D_t}$. Consequently,
\begin{equation}
\|z_t^\star-\hat z_t\|_\infty\le\Gamma_{\rm layer}\sqrt{D_t},
\qquad
\Gamma_{\rm layer}=L_{\rm head}B\sum_{\ell=1}^{L}
b_\ell\kappa_\ell\prod_{j=\ell+1}^{L}a_j.
\label{eq:app-layer-coefficient}
\end{equation}
This supplies an alternative valid coefficient for the radius and decision bounds under the
layerwise assumptions. Staleness can enter at every layer, and early discrepancies may be
amplified by later layers. Neither coefficient is assumed to be numerically available for a
large pretrained model. The bounds explain why response drift is a relevant controller
signal; they do not provide a calibrated error certificate or prove a monotonic relation
between $D_t$ and actual error at every state.

\subsection{Decision Fidelity and Trajectory Fidelity}
\label{app:decision-proof}

For a deterministic decoder at complete sampler state $\xi$, define the decision margin by
\begin{equation}
m_{\Decoder}(\xi,z)
=\inf_{\delta}\bigl\{\|\delta\|_\infty:
\Decoder(\xi,z+\delta)\ne\Decoder(\xi,z)\bigr\}.
\label{eq:app-decision-margin}
\end{equation}
The convention is $\inf\varnothing=+\infty$. The margin may be zero at a decision boundary.
The next state includes all updated histories and auxiliary variables, not only response
tokens. In the main paper, $m_{\mathcal D}(g,z)$ suppresses this auxiliary state.

\begin{proof}[Proof of the decision claim in Proposition~\ref{prop:state-fidelity}]
At a shared state $\xi_t$, Eq.~\eqref{eq:app-radius-bound} implies
\begin{equation}
\Gamma\sqrt{D_t}<m_{\Decoder}(\xi_t,z_t^\star)
\quad\Longrightarrow\quad
\Decoder(\xi_t,\hat z_t)=\Decoder(\xi_t,z_t^\star).
\label{eq:app-decision-condition}
\end{equation}
Indeed, the strict inequality places $\hat z_t$ inside the open $\ell_\infty$ ball of radius
$m_{\Decoder}(\xi_t,z_t^\star)$ around the eager logits. By the definition of the margin,
no perturbation inside this ball changes the complete next sampler state.
\end{proof}

\begin{corollary}[Trajectory preservation]
\label{cor:app-trajectory}
From a common initial sampler state, if the margin condition in
Eq.~\eqref{eq:theory-decision} holds at every reused step of a deterministic run, Archer and
eager decoding produce the same sampler trajectory and completion.
\end{corollary}

\begin{proof}
The two methods begin from the same state. The decision claim in
Proposition~\ref{prop:state-fidelity} preserves equality at each reused step, while a refresh
evaluates the same full model at the same state. Induction over decoder steps proves equality
of all sampler states, including termination and the final completion.
\end{proof}

Failure of the sufficient condition need not change the action. Moreover, a decoder that
stores real-valued confidence histories may have a zero margin for its complete-state
transition even when its token action is stable. The trajectory guarantee is therefore
conditional and can be stronger than action agreement. The shadow-forward study measures
acceptance/re-masking and next-response disagreement as operational diagnostics, rather than
certifying complete-state equality or interpreting logit error alone as failure.

\subsection{Temporal Anchoring as Delayed Feedback}
\label{app:feedback-proof}

\begin{proof}[Proof of Proposition~\ref{prop:delayed-feedback}]
Consider the continuous relaxation $y=E(g)$, with a differentiable extension of the
prompt-cache map and response forward to a neighborhood of the compared embeddings.
For a fixed anchor $y_r$, define
\begin{equation}
\begin{aligned}
z^{\rm fresh}(y)&=\Phi_\theta(y;\Cprompt(y)),\\
z_r^{\rm anchor}(y)&=\Phi_\theta(y;\Cprompt(y_r)).
\end{aligned}
\label{eq:app-logit-maps}
\end{equation}
The chain rule gives
\begin{equation}
\begin{aligned}
J^{\rm fresh}(y)
&=\partial_y\Phi_\theta(y;\Cprompt(y))
+\partial_{\mathcal C}\Phi_\theta(y;\Cprompt(y))
\,\partial_y\Cprompt(y),\\
J_r^{\rm anchor}(y)
&=\partial_y\Phi_\theta(y;\Cprompt(y_r)).
\end{aligned}
\label{eq:app-jacobians}
\end{equation}
The anchored cache is constant with respect to $y$, so its prompt-mediated derivative is zero.
The direct response derivative remains. At $y=y_r$, both derivatives of $\Phi_\theta$ are
evaluated at the same response and cache. Defining
\begin{equation}
\begin{gathered}
F_r=\partial_{\mathcal C}\Phi_\theta(y_r;\Cprompt(y_r))
\,\partial_y\Cprompt(y_r),\\
J^{\rm fresh}(y_r)-J_r^{\rm anchor}(y_r)=F_r,
\end{gathered}
\label{eq:app-feedback-term}
\end{equation}
identifies the prompt-mediated path withheld by anchoring. Away from $y_r$, the direct terms
are evaluated at different caches; the total Jacobian difference need not equal this feedback
term alone.

A full refresh evaluates $z^{\rm fresh}(y)$, replaces the prompt cache with $\Cprompt(y)$,
and sets the new anchor to $y$. The cache discrepancy at that state is zero. Subsequent reused
steps again differentiate the fixed-anchor map. The proposition concerns local forward
sensitivity within a branch, not differentiability of the discrete refresh rule or decoder.
\end{proof}

\paragraph{Conditional effect on reinforcement.}
Let $v$ be a unit response-embedding direction associated with a provisional error and let $w$
be a unit logit direction. At the shared anchor, their directional gains differ by
\begin{equation}
w^\top\!\left[J^{\rm fresh}(y_r)-J_r^{\rm anchor}(y_r)\right]v
=w^\top F_r v.
\label{eq:app-feedback-gain}
\end{equation}
If $w^\top F_r v>0$, anchoring removes a positive contribution to this local directional
gain. This gives a precise sense in which prompt reuse can reduce immediate reinforcement
along that direction while leaving response revision active. It does not imply a smaller
Jacobian norm, global dynamical stability, or higher task accuracy. If the suppressed path
carries useful new evidence, anchoring can instead delay a beneficial update.

The observed non-monotone quality curve is consistent with these two possibilities. Moderate
$K$ permits temporary anchoring followed by synchronization; $K\rightarrow\infty$ removes
the subsequent resets. Propositions~\ref{prop:state-fidelity} and~\ref{prop:delayed-feedback}
thus describe complementary aspects of the same mechanism: state distance bounds the
approximation, while anchoring delays one feedback path. A state radius bounds response drift,
not the elapsed duration of a cache interval.

\subsection{Scope of the Guarantees}

Proposition~\ref{prop:rollback-boundary} preserves the decoder's revision domain;
Proposition~\ref{prop:computational-gain} quantifies reuse under an explicit cost model;
Proposition~\ref{prop:state-fidelity} bounds local logit error and gives a sufficient margin
condition for decision and trajectory fidelity; and Proposition~\ref{prop:delayed-feedback}
identifies the prompt-mediated path withheld by anchoring. These results do not assert global
Lipschitz constants for a pretrained DLM, a calibrated online decision certificate, or an
accuracy improvement from staleness. Equalities concern mathematical forwards; floating-point
differences between execution paths are outside the cache-error bound. The empirical analyses
test the model-dependent effects on revision, next-update disagreement, quality, and latency.

\section{Qualitative Demonstrations and Failure Cases}
\label{app:qualitative}

\begin{figure}[!t]
    \centering
    \includegraphics[width=\textwidth]{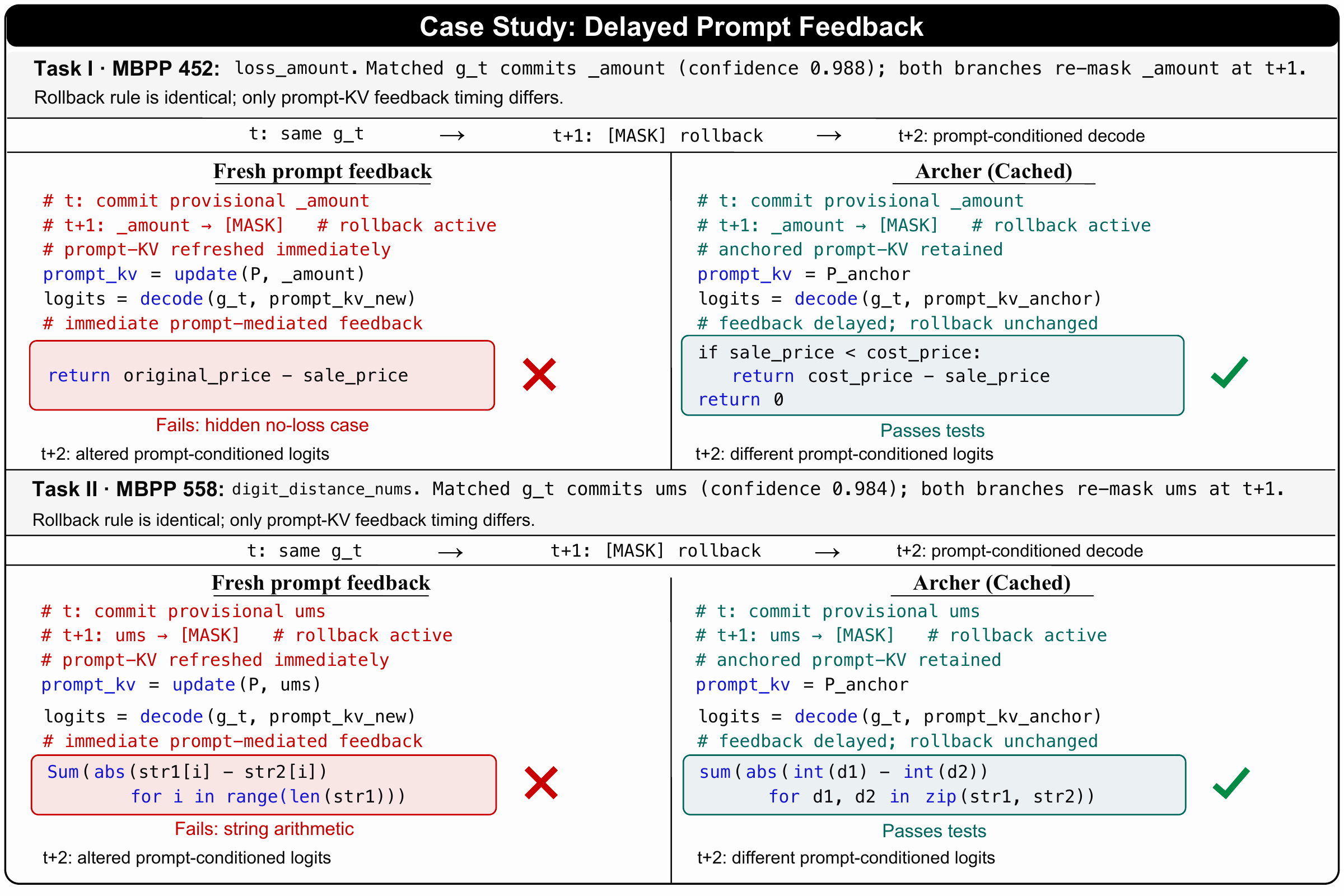}
    \caption{Supplementary delayed-feedback cases from the earlier high-confidence
    ($p\geq0.9$) intervention, distinct from Figure~\ref{fig:feedback-statistics}.}
    \label{fig:app-feedback-cases}
\end{figure}

\subsection{Successful Correction Cases}

These supplementary cases come from an earlier paired intervention that selected the
first token accepted with $p\geq0.9$, rather than the lexical-event protocol in
Section~\ref{sec:feedback-probe}. They are not included in the new quantitative results.
We manually inspected branch-exclusive outcomes whose probe lies inside executable code;
the accompanying figure illustrates two Cached-only successes under that earlier protocol.

In Task 452 (\texttt{loss\_amount}), Fresh settles on an unconditional
\texttt{original\_price - sale\_price} return, which fails whenever the sale
does not produce a loss. Cached instead restores the necessary conditional and
returns zero in the no-loss case. In Task 558 (\texttt{digit\_distance\_nums}),
Fresh applies \texttt{abs} directly to string characters, whereas Cached
converts paired digits to integers before subtraction. Both high-confidence
probed fragments are re-masked within the next update in both branches; the
final difference is therefore not caused by making a token immutable.

\subsection{Failure Case and Interpretation}

The paired intervention also contains the complementary boundary case, Task 760,
where Fresh reaches the correct set-cardinality test and Cached adopts an
incorrect duplicate detector. Delayed feedback is thus an inductive bias, not
a free accuracy guarantee. Moderate anchoring can prevent transient evidence
from being amplified too quickly, but it can also postpone useful evidence.

\section{Refresh-Radius Sensitivity}
\label{app:k-sensitivity}

The main paper reports the operating points that best expose Archer's
quality--latency trade-off. This section supplies the complete refresh-radius
sweep behind those operating points. No radius is selected from a hidden
test-only search.

The refresh radius controls how far the response may move from the state at
which prompt K/V was constructed. The limiting cases have direct
interpretations. At $K=1$, every response change invalidates the snapshot and
Archer approaches eager Saber. At $K=\infty$, the initial prompt snapshot is
never synchronized again.

\begin{table}[!t]
\centering
\ArcherAppendixTableStyle
\setlength{\tabcolsep}{1pt}

\begin{ArcherFitTable}
\begin{tabular}{@{\extracolsep{\fill}}l*{9}{r}@{}}
\toprule
& \multicolumn{3}{c}{MBPP} & \multicolumn{3}{c}{LiveCodeBench} & \multicolumn{3}{c}{HumanEval} \\
\cmidrule(lr){2-4}\cmidrule(lr){5-7}\cmidrule(l){8-10}
Setting & Pass@1 & Time & Spd. & Pass@1 & Time & Spd. & Pass@1 & Time & Spd. \\
\midrule
Saber      & 44.03 & 8.45 & 1.00$\times$ & 7.25 & 17.75 & 1.00$\times$ & 44.51 & 12.21 & 1.00$\times$ \\
$K=1$      & 44.03 & 8.41 & 1.00$\times$ & 7.25 & 17.82 & 1.00$\times$ & 44.51 & 12.14 & 1.01$\times$ \\
$K=2$      & 44.26 & 8.38 & 1.01$\times$ & 7.25 & 17.84 & 1.00$\times$ & 44.51 & 12.18 & 1.00$\times$ \\
$K=4$      & \textbf{44.96} & 7.10 & 1.19$\times$ & 8.00 & 12.77 & 1.39$\times$ & 42.68 & 9.20 & 1.33$\times$ \\
$K=8$      & \textbf{44.96} & 6.48 & 1.30$\times$ & \textbf{9.50} & 10.43 & 1.70$\times$ & 43.29 & 8.07 & 1.51$\times$ \\
$K=16$     & 43.33 & 5.94 & 1.42$\times$ & 9.25 & 8.92 & 1.99$\times$ & 41.46 & 7.07 & 1.73$\times$ \\
$K=32$     & 42.62 & \textbf{5.89} & \textbf{1.43}$\times$ & 8.00 & \textbf{7.68} & \textbf{2.31}$\times$ & 40.24 & 6.70 & 1.82$\times$ \\
$K=\infty$ & 42.62 & 5.98 & 1.41$\times$ & 7.00 & 8.01 & 2.22$\times$ & 39.63 & \textbf{6.69} & \textbf{1.83}$\times$ \\
\bottomrule
\end{tabular}
\end{ArcherFitTable}
\caption{Refresh-radius sweep on LLaDA-8B-Instruct. Time is seconds per problem.}
\label{tab:app-k-sensitivity}
\end{table}

Table~\ref{tab:app-k-sensitivity} reports the full logarithmic sweep on all
three code benchmarks. Latency improves as refreshes become less frequent, whereas quality is
non-monotone.  Moderate radii preserve enough synchronization to avoid
long-lived context error while still delaying prompt-mediated reinforcement.
The deterioration at large radii is therefore not an unexplained tuning
artifact; it is the empirical counterpart of the approximation term in
Eq.~\eqref{eq:app-k-bound}.  The sweep also makes clear that the same radius
need not optimize every benchmark, which motivates reporting per-task
operating points rather than a universal ``best'' $K$.

\section{Rollback Activity}
\label{app:rollback-preservation}

We compare response revisions under No-Rollback, Saber, and Archer. The
measurement protocol is given in Appendix~\ref{app:rollback-protocol}.

\begin{table}[!t]
\centering
\ArcherAppendixTableStyle
\setlength{\tabcolsep}{2pt}
\begin{ArcherFitTable}
\begin{tabular}{@{\extracolsep{\fill}}l*{6}{r}@{}}
\toprule
\multirow{2}{*}{Method}
& \multicolumn{2}{c}{MBPP}
& \multicolumn{2}{c}{LiveCodeBench}
& \multicolumn{2}{c}{HumanEval} \\
\cmidrule(lr){2-3}\cmidrule(lr){4-5}\cmidrule(l){6-7}
& Rev. (\%) & Updates & Rev. (\%) & Updates & Rev. (\%) & Updates \\
\midrule
No-Rollback
& 0.00 & 65.97 & 0.00 & 74.73 & 0.00 & 66.26 \\
Saber
& 29.18 & 146.98 & \textbf{28.21} & 158.83 & 27.87 & 157.55 \\
\textbf{Archer}
& \textbf{30.03} & 147.51 & 26.93 & 163.25 & \textbf{30.24} & 158.49 \\
\bottomrule
\end{tabular}
\end{ArcherFitTable}
\caption{Rollback activity with $K=8$. Rev. is the percentage of response positions revised
after their first non-mask prediction; Updates is the mean number of response updates.
No-Rollback provides the zero-revision control.}
\label{tab:app-rollback-preservation}
\end{table}

Table~\ref{tab:app-rollback-preservation} shows that Archer retains substantial revision
activity and does not shorten the
trajectory relative to Saber. Its speedup therefore does not come from
freezing response tokens or removing opportunities for correction. This
measurement supports Proposition~\ref{prop:rollback-boundary} at the realized
trajectory level.

\section{Generalization Beyond Code}
\label{app:generalization}

\begin{table}[!t]
\centering
\ArcherAppendixTableStyle
\setlength{\tabcolsep}{2.5pt}
\begin{ArcherFitTable}
\begin{tabular}{@{\extracolsep{\fill}}l*{9}{r}@{}}
\toprule
\multirow{2}{*}{Method}
& \multicolumn{3}{c}{MATH-500}
& \multicolumn{3}{c}{FinQA}
& \multicolumn{3}{c}{PubMedQA} \\
\cmidrule(lr){2-4}\cmidrule(lr){5-7}\cmidrule(l){8-10}
& Acc.$\uparrow$ & Time$\downarrow$ & Spd.$\uparrow$
& Acc.$\uparrow$ & Time$\downarrow$ & Spd.$\uparrow$
& Acc.$\uparrow$ & Time$\downarrow$ & Spd.$\uparrow$ \\
\midrule
LLaDA-Confidence & 24.00 & 12.60 & 1.00$\times$ & 28.51 & 42.65 & 1.00$\times$ & 67.40 & 21.26 & 1.00$\times$ \\
Saber & \textbf{25.60} & 8.13 & 1.55$\times$ & 29.56 & 24.66 & 1.73$\times$ & 69.40 & 14.15 & 1.50$\times$ \\
\textbf{Archer} & 25.00 & \textbf{6.39} & \textbf{1.97}$\times$ & \textbf{30.17} & \textbf{9.86} & \textbf{4.33}$\times$ & \textbf{69.50} & \textbf{9.58} & \textbf{2.22}$\times$ \\
\bottomrule
\end{tabular}
\end{ArcherFitTable}
\par\vspace{6pt}
\begin{ArcherFitTable}
\begin{tabular}{@{\extracolsep{\fill}}l*{8}{r}@{}}
\toprule
\multirow{2}{*}{Method}
& \multicolumn{3}{c}{ARC-Challenge}
& \multicolumn{3}{c}{MMLU-Pro}
& \multicolumn{2}{c}{Overall} \\
\cmidrule(lr){2-4}\cmidrule(lr){5-7}\cmidrule(l){8-9}
& Acc.$\uparrow$ & Time$\downarrow$ & Spd.$\uparrow$
& Acc.$\uparrow$ & Time$\downarrow$ & Spd.$\uparrow$
& Perf.$\uparrow$ & Spd.$\uparrow$ \\
\midrule
LLaDA-Confidence & 86.09 & 12.54 & 1.00$\times$ & 30.60 & 18.46 & 1.00$\times$ & 47.32 & 1.00$\times$ \\
Saber & \textbf{86.18} & 11.73 & 1.07$\times$ & 30.42 & 12.21 & 1.51$\times$ & 48.23 & 1.47$\times$ \\
\textbf{Archer} & 85.75 & \textbf{9.12} & \textbf{1.38}$\times$ & \textbf{31.44} & \textbf{7.58} & \textbf{2.43}$\times$ & \textbf{48.37} & \textbf{2.47}$\times$ \\
\bottomrule
\end{tabular}
\end{ArcherFitTable}
\caption{Non-code evaluation with LLaDA-8B-Instruct. Accuracy is in percent and Time in
seconds per example; Spd. is relative to the Base decoder, LLaDA-Confidence. MMLU-Pro pools
the Engineering and Physics subsets (2,268 questions). Overall reports macro-average accuracy
and mean per-benchmark speedup over the five benchmarks.}
\label{tab:app-cross-domain}
\end{table}

Archer's controller observes response-state changes and uses neither code
structure nor execution feedback. We therefore evaluate whether prompt-side
reuse transfers to reasoning tasks with different answer formats. The
extended suite comprises MATH-500 \citep{hendrycks2021math,lightman2024verify}, FinQA
\citep{chen2021finqa}, PubMedQA \citep{jin2019pubmedqa}, ARC-Challenge
\citep{clark2018arc}, and MMLU-Pro \citep{wang2024mmlupro}, covering
mathematical reasoning, financial reasoning, biomedical question answering,
science questions, and subject-specific knowledge and reasoning.

Archer reduces latency relative to LLaDA-Confidence on every benchmark,
with speedups from $1.38\times$ to $4.33\times$. The largest reduction
occurs on FinQA, where latency falls from 42.65 to 9.86 seconds while
numerical-answer accuracy increases from $28.51\%$ to $30.17\%$.
PubMedQA improves by $2.10$ accuracy points at $2.22\times$ speedup.
On the pooled MMLU-Pro evaluation, Archer raises accuracy from $30.60\%$
to $31.44\%$ and reduces mean latency from 18.46 to 7.58 seconds,
yielding $2.43\times$ speedup.

Comparison with the rollback-capable Saber isolates the quality effect of
prompt reuse. Archer improves accuracy by $0.61$ points on FinQA, $0.10$
points on PubMedQA, and $1.02$ points on MMLU-Pro. The benefit is not
uniform: accuracy decreases by $0.60$ points on MATH-500 and $0.43$ points
on ARC-Challenge relative to Saber. These outcomes are consistent with
bounded staleness delaying either premature reinforcement or useful new
context, depending on the task and trajectory.

Across the five benchmarks, macro-average accuracy is $48.37\%$ for
Archer, $48.23\%$ for Saber, and $47.32\%$ for LLaDA-Confidence. The mean
per-benchmark speedups over LLaDA-Confidence are $2.47\times$ for Archer
and $1.47\times$ for Saber. Thus, the prompt-cache controller transfers
beyond program synthesis and reduces decoding cost across all five
evaluated tasks, while its effect on accuracy remains task-dependent.

\section{Prompt-Length Scaling}
\label{app:prompt-scaling}
\label{app:scaling}

Archer removes repeated prompt projections while recomputing the revisable
response. We test this systems claim through realized cache use and prompt-
and response-length scaling. At the reported MBPP and LiveCodeBench operating
points, 81.62\% and 80.57\% of main-trajectory model calls, respectively, use
the cached prompt path. Archer may execute slightly more NFEs than Saber, so
the speedup comes from lower work per update rather than a shorter correction
trajectory.

The matched-state timing protocol is specified in
Appendix~\ref{app:scaling-protocol}.

\begin{table}[!t]
\centering
\ArcherAppendixTableStyle
\setlength{\tabcolsep}{1pt}

\begin{ArcherFitTable}
\begin{tabular}{@{\extracolsep{\fill}}l*{6}{rr}@{}}
\toprule
\multirow{2}{*}{Update} &
\multicolumn{2}{c}{$P=128$} &
\multicolumn{2}{c}{$P=256$} &
\multicolumn{2}{c}{$P=512$} &
\multicolumn{2}{c}{$P=1024$} &
\multicolumn{2}{c}{$P=2048$} &
\multicolumn{2}{c}{Overall} \\
\cmidrule(lr){2-3}\cmidrule(lr){4-5}\cmidrule(lr){6-7}
\cmidrule(lr){8-9}\cmidrule(lr){10-11}\cmidrule(lr){12-13}
& Time$\downarrow$ & Spd.$\uparrow$
& Time$\downarrow$ & Spd.$\uparrow$
& Time$\downarrow$ & Spd.$\uparrow$
& Time$\downarrow$ & Spd.$\uparrow$
& Time$\downarrow$ & Spd.$\uparrow$
& Time$\downarrow$ & Spd.$\uparrow$ \\
\midrule
Full & 46.11 & 1.00$\times$ & 57.80 & 1.00$\times$ & 82.88 & 1.00$\times$ & 145.66 & 1.00$\times$ & 271.91 & 1.00$\times$ & 120.87 & 1.00$\times$ \\
Cached & \textbf{35.41} & \textbf{1.30$\times$} & \textbf{36.39} & \textbf{1.59$\times$} & \textbf{40.04} & \textbf{2.07$\times$} & \textbf{46.27} & \textbf{3.15$\times$} & \textbf{58.84} & \textbf{4.62$\times$} & \textbf{43.39} & \textbf{2.55$\times$} \\
\bottomrule
\end{tabular}
\end{ArcherFitTable}
\caption{Prompt-length scaling with $G=256$ and $K=8$. Time is in milliseconds.}
\label{tab:app-prompt-scaling}
\end{table}

Table~\ref{tab:app-prompt-scaling} shows that cached-step speedup increases monotonically from
$1.30\times$ to $4.62\times$.
Full-forward cost grows rapidly with the prompt, whereas cached-step time is
dominated by the fixed response region. The measured trend directly matches
the complexity analysis: prompt reuse becomes more valuable as the immutable
context occupies a larger fraction of the sequence.

\section{Response-Length Scaling}
\label{app:response-scaling}

The response-budget sweep fixes the prompt length at $P=512$ and the refresh
radius at $K=8$, and varies the response budget over $G\in\{256,512,1024\}$.
We measure complete realized trajectories, including cache construction and
refreshes (Appendix~\ref{app:scaling-protocol}).

\begin{table}[!t]
\centering
\ArcherAppendixTableStyle
\setlength{\tabcolsep}{1pt}
\begin{ArcherFitTable}
\begin{tabular}{@{\extracolsep{\fill}}l*{3}{rr}@{}}
\toprule
\multirow{2}{*}{Method}
& \multicolumn{2}{c}{$G=256$}
& \multicolumn{2}{c}{$G=512$}
& \multicolumn{2}{c}{$G=1024$} \\
\cmidrule(lr){2-3}\cmidrule(lr){4-5}\cmidrule(l){6-7}
& Time$\downarrow$ & Spd.$\uparrow$
& Time$\downarrow$ & Spd.$\uparrow$
& Time$\downarrow$ & Spd.$\uparrow$ \\
\midrule
Saber
& 19.31 & 1.00$\times$
& 49.00 & 1.00$\times$
& 154.38 & 1.00$\times$ \\
\textbf{Archer}
& \textbf{11.57} & \textbf{1.67$\times$}
& \textbf{33.63} & \textbf{1.46$\times$}
& \textbf{94.34} & \textbf{1.64$\times$} \\
\bottomrule
\end{tabular}
\end{ArcherFitTable}
\caption{End-to-end response-length scaling with $P=512$ and $K=8$.
Time is in seconds per generated sample; Spd. is relative to Saber.}
\label{tab:app-response-scaling}
\end{table}

Table~\ref{tab:app-response-scaling} shows that Archer remains faster at every
tested response length, with gains of $1.46$--$1.67\times$ over Saber.
Both methods take longer as the response budget grows, while prompt reuse
continues to reduce end-to-end latency. The speedup varies across the tested
response budgets.

\section{Additional Results on Delayed Prompt Feedback}
\label{app:feedback}

We extend the matched-state intervention in Section~\ref{sec:feedback-probe}, using
all 427 MBPP problems, LLaDA-8B-Instruct, and $K=11$. Appendix~\ref{app:feedback-protocol}
gives event selection, branch construction, and statistical methods;
Appendix~\ref{app:mechanism-timing} specifies the independent timing protocol.

\paragraph{Event coverage.}
The protocol triggers on 393 problems
($92.04\%$); the other 34 retain the same completion in both branches.
The selected events occur at a mean update of 34.24 (median 34), with mean response
progress $25.63\%$ and mean acceptance confidence $83.66\%$. Their anchor distance
has median 6 and ranges from 1 to 10. These describe the sampled events rather than
all token acceptances.

\subsection{Final Quality and Paired Outcomes}
Pass@1 evaluates the same completion on the Base and ET suites, with all 427 problems in
the denominator. Branch-exclusive successes count problems passed by one branch but not
the other, separately for each suite. Table~\ref{tab:feedback-paired-results} reports
the complete paired outcomes. Delayed-only successes are twice as numerous on Base and
$2.5\times$ as numerous on ET, yielding seven and six additional passed programs,
respectively. The final completions differ on 188 problems ($44.03\%$), showing that
one early refresh can change the subsequent generation trajectory.

\begin{table}[!htbp]
\centering
\ArcherAppendixTableStyle
\setlength{\tabcolsep}{4pt}
\begin{ArcherFitTable}
\begin{tabular}{@{\extracolsep{\fill}}lrr@{}}
\toprule
Outcome & Base & ET \\
\midrule
Immediate passed & 183 (42.86\%) & 133 (31.15\%) \\
Delayed passed & \textbf{190 (44.50\%)} & \textbf{139 (32.55\%)} \\
\midrule
Both passed & 176 & 129 \\
Delayed only & \textbf{14} & \textbf{10} \\
Immediate only & 7 & 4 \\
Neither passed & 230 & 284 \\
\midrule
Delayed gain (pp) & $+1.64$ & $+1.40$ \\
$95\%$ CI (pp) & $[-0.47,3.75]$ & $[-0.23,3.28]$ \\
McNemar $p$ & 0.189 & 0.180 \\
\bottomrule
\end{tabular}
\end{ArcherFitTable}
\caption{Complete paired quality results for the $K=11$ intervention. Pass rates
use all 427 problems, including 34 without an eligible event. The four paired-outcome
rows partition the same 427 problems in each suite. Gain is delayed minus immediate;
intervals use paired problem resampling.}
\label{tab:feedback-paired-results}
\end{table}

Both quality intervals include zero. The positive point estimates therefore
provide local evidence that normal delayed feedback can preserve final quality, without
establishing a statistically resolved advantage. The interpretation is conditional on
the specified event-selection protocol. At the first post-fork update,
Fresh uses logits from the full cache-building forward, whereas Cached uses a cached
response forward. Thus, the comparison includes their numerical-path difference and
cannot attribute the outcome solely to prompt-feedback timing.

\subsection{Revision through Termination}
Table~\ref{tab:feedback-target-results} follows the same selected token over five
updates and through termination. Short-window re-masking is identical in aggregate:
38 targets are re-masked in both branches and 11 only in each branch. Over the full
remaining trajectory, delayed feedback still re-masks 136 targets ($34.61\%$), and
131 targets reproduce the originally accepted token after re-masking. Revision and
re-prediction thus remain active beyond the short window. Token-identity changes are
much rarer than completion differences; they do not identify semantic corrections
or explain all program-level quality differences.

\begin{table}[!htbp]
\centering
\ArcherAppendixTableStyle
\setlength{\tabcolsep}{3pt}
\begin{ArcherFitTable}
\begin{tabular}{@{\extracolsep{\fill}}lrr@{}}
\toprule
Target behavior & Immediate & Delayed \\
\midrule
\multicolumn{3}{l}{\emph{First five updates}} \\
Re-masked & 49 (12.47) & 49 (12.47) \\
Other non-mask identity & 0 (0.00) & 0 (0.00) \\
MASK at update 5 & 29 (7.38) & 30 (7.63) \\
First re-mask (updates) & 2.86 & 3.18 \\
\midrule
\multicolumn{3}{l}{\emph{Complete remaining trajectory}} \\
Re-masked & 144 (36.64) & 136 (34.61) \\
Original token reappeared & 140 (35.62) & 131 (33.33) \\
Other final non-mask token & 5 (1.27) & 8 (2.04) \\
\bottomrule
\end{tabular}
\end{ArcherFitTable}
\caption{Target revision for 393 eligible $K=11$ events. Cells give count
(percentage), except first re-mask time, averaged over the 49 targets re-masked
within five updates in each branch. ``Original token reappeared'' requires a preceding
re-mask; rows can overlap. Identity is compared with the token at the branch point.}
\label{tab:feedback-target-results}
\end{table}

The first re-mask means condition on each branch's own revised targets, so they do
not establish which branch corrects errors faster.

\subsection{Generation Time and Model Calls}

The repeat means are 8.3539, 8.3854, and 8.1810 seconds for Saber, and 6.3171, 6.3142, and
6.1952 seconds for Archer, giving $8.3067\pm0.1101$ and $6.2755\pm0.0695$ seconds.
Speedup is the ratio of these mean latencies ($1.32\times$), using Saber as the reference.
This repeated measurement is separate from the single-run timings in
Table~\ref{tab:main-results}. Every Archer repeat matches the cache-diagnostic main
trajectory in final output, update count, and main model-call count on all problems.
The mean model-call counts are 148.87 for Saber and 150.50 for Archer, so the measured
saving does not rely on shortening the decoding trajectory.
Archer averages 27.66 full cache-building forwards and 122.84 cached forwards per
problem; cached forwards account for $81.62\%$ of its main model calls. Its mean
latency is $24.45\%$ lower despite a slightly larger model-call count, connecting
the measured speedup to reuse of prompt computation.


\section{Additional Results on State-Aware Refresh}
\label{app:cache-validity}

We extend Section~\ref{sec:validity-signal} with 5,330 same-path old-cache versus
fresh-cache comparisons from all 427 MBPP problems at $K=11$. We report next-update
disagreement and mean differing response-token counts. These diagnose changes in
decoder behavior, rather than semantic errors. Sampling, weighting, and statistical
methods are specified in Appendix~\ref{app:cache-validity-protocol}.

\subsection{Distance, Age, and the Refresh Boundary}
Table~\ref{tab:cache-groups} gives the underlying distance groups together with
age groups and the real controller's reuse/refresh partition. Both disagreement
incidence and mean token differences increase across distance groups. Age groups
are not monotonic: even $A_t\leq1$ gives $51.07\%$ disagreement, compared with
$48.56\%$ at ages 2--4. Elapsed updates alone therefore do not order these sampled
states by decision sensitivity.

\begin{table}[!htbp]
\centering
\ArcherAppendixTableStyle
\setlength{\tabcolsep}{4pt}
\begin{ArcherFitTable}
\begin{tabular}{@{\extracolsep{\fill}}lrrr@{}}
\toprule
State group & $N$ & \shortstack{Disagreement\\(\%)} & \shortstack{Token\\differences} \\
\midrule
\multicolumn{4}{l}{\emph{State distance $D_t$}} \\
0--2 & 308 & 25.00 & 0.6266 \\
3--5 & 1,447 & 44.30 & 1.0684 \\
6--8 & 1,729 & 51.82 & 1.2915 \\
9--10 & 875 & 57.03 & 1.4709 \\
$\geq11$ & 971 & 68.80 & 2.5118 \\
\midrule
\multicolumn{4}{l}{\emph{Cache age $A_t$ (updates)}} \\
0--1 & 1,024 & 51.07 & 1.7637 \\
2--4 & 2,123 & 48.56 & 1.2270 \\
5--8 & 1,897 & 56.88 & 1.5271 \\
9--11 & 286 & 51.75 & 1.3636 \\
\midrule
\multicolumn{4}{l}{\emph{Controller decision}} \\
Reuse ($D_t<11$) & 4,359 & 48.47 & 1.2065 \\
Refresh ($D_t\geq11$) & 971 & 68.80 & 2.5118 \\
\bottomrule
\end{tabular}
\end{ArcherFitTable}
\caption{Same-path cache comparisons at $K=11$. Each grouping partitions the
same 5,330 observations. $N$ counts observations; rates and mean token differences
pool observations within each row. All comparisons evaluate old-cache reuse before
the controller acts, including states where it will refresh.}
\label{tab:cache-groups}
\end{table}

At the actual refresh boundary, disagreement is $20.33$ points higher than in the
reuse region, and mean token differences are $2.08\times$ as large. This supports
refreshing when the response has moved sufficiently far from its snapshot; it does
not establish that 11 is an optimal threshold. The lowest distance group includes
39 initial fresh-cache states with zero same-path difference. Removing them leaves
269 states at $D_t=1$--$2$ with $28.62\%$ disagreement, still below the larger-distance
groups. The distance trend is therefore not solely an artifact of initial zero-age
observations.

\subsection{Equal-Budget Signal Comparison}

Table~\ref{tab:cache-equal-budget} reports both the flagged $20\%$ and the remaining
$80\%$, including the common unranked baseline. Distance concentrates more disagreement
in the flagged states ($68.70\%$ vs.\ $59.59\%$) and more token differences (2.5732
vs.\ 1.6429), leaving lower disagreement and fewer token differences in the remaining
states at an identical selection budget.

\begin{table}[!htbp]
\centering
\ArcherAppendixTableStyle
\setlength{\tabcolsep}{4pt}
\begin{ArcherFitTable}
\begin{tabular}{@{\extracolsep{\fill}}lrr@{}}
\toprule
Outcome & Cache age & State distance \\
\midrule
\multicolumn{3}{l}{\emph{All candidate states}} \\
Disagreement (\%) & 54.35 & 54.35 \\
Token differences & 1.5552 & 1.5552 \\
\midrule
\multicolumn{3}{l}{\emph{Flagged $20\%$}} \\
Disagreement (\%) & 59.59 & \textbf{68.70} \\
Token differences & 1.6429 & \textbf{2.5732} \\
\midrule
\multicolumn{3}{l}{\emph{Remaining $80\%$}} \\
Disagreement (\%) & 53.04 & \textbf{50.76} \\
Token differences & 1.5333 & \textbf{1.3007} \\
\bottomrule
\end{tabular}
\end{ArcherFitTable}
\caption{Equal-budget offline ranking at $K=11$, using 5,291 positive-age
observations from 427 problems. Each cell is the equal-weight mean of the per-problem
weighted rate or token count; tied scores receive fractional selection weights.
Higher flagged and lower residual values indicate better concentration of cache-sensitive
states. The common baseline is shown for both scores.}
\label{tab:cache-equal-budget}
\end{table}

The age-minus-distance
reduction is $2.28$ points for disagreement ($95\%$ interval $[1.63,2.97]$), and $0.2326$
tokens for next-state Hamming distance ($[0.1748,0.2942]$). The latter changes from 1.5333
to 1.3007 tokens, a $15.17\%$ relative reduction. This is an offline ranking diagnostic
on a common trajectory; neither score is deployed as a new online refresh policy.
Consequently, it measures signal quality rather than end-to-end accuracy or speed gains.

\subsection{Association after Controlling the Other Signal}

Table~\ref{tab:cache-partial-correlations} reports partial Spearman correlations
controlling for the other refresh signal. Distance retains positive associations with decision disagreement and
token differences after controlling age. Age has slightly negative associations
after controlling distance. These are descriptive associations.

\begin{table}[!htbp]
\centering
\ArcherAppendixTableStyle
\setlength{\tabcolsep}{4pt}
\begin{ArcherFitTable}
\begin{tabular}{@{\extracolsep{\fill}}lrr@{}}
\toprule
Signal (control) & \shortstack{Decision\\disagreement} & \shortstack{Token\\differences} \\
\midrule
State distance (age) & 0.216 & 0.268 \\
Cache age (distance) & $-0.064$ & $-0.102$ \\
\bottomrule
\end{tabular}
\end{ArcherFitTable}
\caption{Descriptive partial Spearman correlations at $K=11$. Each row
controls for the other refresh signal; outcomes use the same-path comparison.}
\label{tab:cache-partial-correlations}
\end{table}


\FloatBarrier
\section{Experimental Details and Protocols}
\label{app:experimental-setup}

This section specifies the evaluation protocol used throughout the paper. The
main experiments follow the code-generation setting of Saber, while the
supplementary cross-domain study changes only the task prompt and evaluator.
Unless noted otherwise, we decode each problem once without demonstrations
or test-time sampling.

\subsection{Datasets}
\label{app:datasets}

Our primary comparison uses MBPP \citep{austin2021mbpp}, HumanEval
\citep{chen2021humaneval}, their Extended Test Cases (ET) versions
\citep{dong2025codescore}, and LiveCodeBench \citep{jain2025livecodebench}.
MBPP contains 427 sanitized Python synthesis problems and HumanEval contains
164 function-completion tasks. Their ET versions preserve the original
problems but evaluate generated programs against additional edge-case tests.
LiveCodeBench is contamination-aware; we use all 400 tasks in
\texttt{release\_v1}. In each code benchmark, the model receives the task
description and, where applicable, the function signature, public tests, or
starter code. The model must produce one executable Python completion that is
evaluated without manual repair.

To evaluate transfer beyond code, we additionally use the 500-problem
MATH-500 test set derived from MATH \citep{hendrycks2021math,lightman2024verify}, the canonical
FinQA test split \citep{chen2021finqa}, and the 1,000-example expert-labeled
PubMedQA set \citep{jin2019pubmedqa}, together with ARC-Challenge
\citep{clark2018arc} and the Engineering and Physics subsets of MMLU-Pro
\citep{wang2024mmlupro}. The two MMLU-Pro subsets contain 969 and 1,299
questions, respectively; their 2,268 questions are pooled for the reported
accuracy and mean latency. MATH-500 requests a boxed final answer;
FinQA requests a numerical answer from a financial report and table; and
PubMedQA requires one of \emph{yes}, \emph{no}, or \emph{maybe} after reading
the question and abstract. ARC-Challenge and MMLU-Pro are multiple-choice
question-answering tasks whose responses identify the selected option.

\subsection{Baselines}

We compare Archer with three classes of DLM decoding methods. The first is
standard confidence-based LLaDA decoding, which executes the full denoising
schedule without KV reuse. The second includes efficient DLM methods that
target standard, non-rollback decoding: Fast-dLLM in cache-only and parallel
forms \citep{wu2025fastdllm}, dKV-Cache in decode and greedy modes
\citep{ma2025dkvcache}, and dLLM-Cache \citep{liu2025dllmcache}. The third is
Saber \citep{dong2026saber}, which performs a full forward pass at every
rollback update. Archer uses exactly Saber's acceptance, replacement, and
re-masking rules; the only difference is whether the prompt K/V state is
reused or refreshed. This pairing isolates the effect of the cache controller
from a change in the rollback decoder.

All baselines use the same task prompt, response budget, base-model revision,
and evaluator whenever their algorithms permit. We follow the released
configurations of Fast-dLLM, dKV-Cache, dLLM-Cache, and Saber. Fast-dLLM is
evaluated both with cache-only transfer and factor-one parallel transfer, and
dKV-Cache-Greedy uses its released random ordering with seed 42.

\subsection{Metrics}
\label{app:metrics}

For code generation, Pass@1 is the fraction of tasks whose single extracted
completion passes every test in the corresponding evaluator:
\begin{equation}
\mathrm{Pass@1} = \frac{1}{N}\sum_{i=1}^{N}
\mathbb{I}\!\left[\mathrm{Passed}(\widehat y_i)\right].
\end{equation}
Here, $N$ is the number of tasks and $\widehat y_i$ is the decoded completion
for task $i$. For MBPP and HumanEval, we compute this score separately on the
original Base test suite and the Extended Test Cases (ET) suite; the latter evaluates
the same completion against additional edge-case tests.
LiveCodeBench Pass@1 is computed by its official evaluator. MATH-500 is
scored after canonical answer normalization, and PubMedQA uses normalized
yes/no/maybe labels. For FinQA, we compare the extracted numerical answer
with the gold \texttt{exe\_ans}; this is numerical-answer accuracy rather
than official program accuracy. ARC-Challenge and MMLU-Pro report answer
accuracy on their multiple-choice questions. MMLU-Pro accuracy and mean
latency are computed over the pooled Engineering and Physics questions,
rather than as an unweighted average of the two subset scores.

We also report the mean number of response updates (Steps), synchronized
generation time per example (Time), and Speedup. Main-table speedups are
computed against confidence-based LLaDA on the same benchmark; the
cross-backbone study instead uses Saber on the same model as its reference.
For Table~\ref{tab:main-results}, Overall performance averages the five Pass@1 scores,
and Overall speedup averages the three per-benchmark latency ratios.
Table~\ref{tab:backbone-results} averages six Pass@1 scores (Base and ET across three
backbones) and three backbone-specific speedups in its Overall columns, separately for
each benchmark.
The non-code comparison uses LLaDA-Confidence as the reference for speedup.
Its Overall accuracy gives each of the five benchmarks equal weight, and
its Overall speedup is the mean of the per-benchmark LLaDA-Confidence/decoder
latency ratios.
Timing includes cache construction, refreshes, cached forwards, and decoder
control, but excludes model loading, tokenization, and program execution.

\subsection{Implementation Details}

The primary results use LLaDA-8B-Instruct \citep{nie2025llada}, pinned to
revision \texttt{6059b30}. The cross-backbone study additionally evaluates
Dream-v0-Instruct-7B \citep{ye2025dream} and DiffuCoder-7B-cpGRPO
\citep{gong2025diffucoder}, preserving each model's native mask identifier,
logit alignment, and cache interface. All experiments use temperature zero,
batch size one, a generation length of 256, and a block length of 256. Saber
and Archer use $n=2$ and $\mu=2$.

The principal Archer operating points use $K=11$ on MBPP, $K=10$ on
LiveCodeBench, and $K=8$ on HumanEval. Cross-backbone results use $K=11/8$
for LLaDA, $17/9$ for Dream, and $8/15$ for DiffuCoder on
MBPP/HumanEval, respectively. Table~\ref{tab:backbone-results} uses the same LLaDA results
and refresh radii as Table~\ref{tab:main-results}. These operating points are accompanied by the
logarithmic sensitivity sweep in Appendix~\ref{app:k-sensitivity}; they do
not assert that a universal radius is optimal. For the five non-code
benchmarks, Archer uses $K=16$ on MATH-500, FinQA, ARC-Challenge, and
MMLU-Pro, and $K=4$ on PubMedQA. All timings are
measured with CUDA synchronization immediately before and after generation.
Section~\ref{app:mechanism} gives the event selection, shadow sampling, equal-budget
analysis, and repeated timing protocols for the two mechanism experiments.

\subsection{Non-Code Evaluation}
\label{app:noncode-protocol}

All comparisons use LLaDA-8B-Instruct with three decoders:
LLaDA-Confidence, Saber, and Archer. Archer uses $K=16$ on MATH-500,
FinQA, ARC-Challenge, and MMLU-Pro, and $K=4$ on PubMedQA.
For MMLU-Pro, we evaluate the Engineering subset (969 questions) and the
Physics subset (1,299 questions). We pool all 2,268 questions into one
benchmark result, with accuracy and mean latency weighted by the number of
questions in each subset. This pooled result appears once in
Table~\ref{tab:app-cross-domain} and contributes once to the suite average.
The scoring conventions are given in Section~\ref{app:metrics}.

\subsection{Controlled-Analysis Protocols}
\label{app:mechanism}

This appendix specifies the feedback intervention and cache-validity diagnostic in
Sections~\ref{sec:feedback-probe} and~\ref{sec:validity-signal}. Both use all 427 MBPP
problems, LLaDA-8B-Instruct revision \texttt{6059b30}, one 256-token response block,
temperature zero, batch size one, $n=\mu=2$, and refresh radius $K=11$.
The $K=8$ rollback-activity measurement is a separate supporting diagnostic.
General prompting and evaluation conventions follow
Sections~\ref{app:datasets} and~\ref{app:metrics}. Additional results appear in
Appendices~\ref{app:feedback} and~\ref{app:cache-validity}.

\subsubsection{Rollback-Activity Measurement}
\label{app:rollback-protocol}

A response position is counted as revised if it is re-masked or replaced
after first receiving a non-mask prediction. Saber and Archer use the same
rollback rule; No-Rollback provides a zero-revision control. We
use the representative radius $K=8$ throughout this diagnostic. The MBPP
measurement uses a fixed 150-problem subset; HumanEval and LiveCodeBench use
their complete evaluation sets.

\subsubsection{Matched-State Feedback Intervention}
\label{app:feedback-protocol}

\paragraph{Event selection.}
We select the first update in each problem that accepts an eligible token from $\Mask$
while normal Archer still permits prompt-cache reuse. Eligible tokens exclude special and
added tokenizer tokens and contain a letter, digit, or underscore. At the selected position
$j$, acceptance confidence, response progress, and anchor state satisfy
\begin{equation}
\begin{gathered}
0.60\leq p_j<0.95,\qquad 0.15\leq u_t\leq0.80,\\
g_r[j]=\Mask,\qquad d_H(g_t,g_r)<11,
\end{gathered}
\end{equation}
where $u_t$ is the fraction of the 256 response positions that are non-mask. The anchor
condition ensures that the newly accepted prediction has not entered the old cache snapshot.
No prior rollback is required. If several tokens qualify at the first eligible update,
we choose the highest-confidence token, breaking ties by the lowest position index.
There is at most one intervention per problem.

\paragraph{Branch construction.}
Immediately after acceptance, we clone the response, accepted-position state, all confidence
histories, prompt cache, and response anchor. \emph{Immediate refresh} (Fresh in the logs)
rebuilds prompt K/V once from the shared response before the next update and resets its
anchor to that response. \emph{Delayed refresh} (Cached) keeps the preceding cache and
anchor. Thereafter, each branch independently follows the normal $D_t\geq11$ refresh
controller and the same Saber acceptance and rollback rules. We neither force the cache
to remain stale past the threshold nor refresh Fresh at every update. The one extra refresh
can change subsequent states and hence subsequent refresh times. We record the target for
five updates and then continue both branches to termination. Cached matches the normal
$K=11$ reference on all 427 final outputs, update counts, and main model-call counts.

\paragraph{Endpoints and statistical analysis.}
Pass@1 evaluates the same completion on the Base and ET suites, with all 427 problems
in the denominator. Branch-exclusive successes count problems passed by one branch
but not the other, separately for each suite. Target re-masking and token-identity
changes use the 393 eligible events as their denominator. We follow the selected
target for five updates and through termination; first re-mask time within five
updates is averaged only over targets re-masked in that window.

We bootstrap the per-problem Cached-minus-Fresh pass indicator with 10,000 paired
resamples (seed 42) and report $95\%$ percentile intervals. Exact McNemar tests are
two-sided. Appendix~\ref{app:feedback} reports the complete outcomes and intervals.

\subsubsection{Decision-Level Cache Validity}
\label{app:cache-validity-protocol}

\paragraph{Sampling and paired computation.}
State distance $D_t=d_H(g_t,g_r)$ counts changed response positions, including changes
between mask and non-mask tokens. Cache age $A_t$ counts response updates since the
cache was built. We sample every 12 updates, with a problem-specific phase drawn uniformly
from $\{0,\ldots,11\}$ using seed $42+i$ for problem index $i$. Sampling occurs before
the refresh decision, covers both reuse and refresh states, and has no per-problem cap.
This yields 5,330 observations, with 4--19 per problem.

At each sampled state, we compare old-cache reuse with a cache rebuilt from the identical
current response. Both alternatives use the same cached response-forward path, so their
prompt cache is the changed input. Each next-step transition starts from a clone of the
response, acceptance state, and confidence histories and uses the same decoder rules.
Shadow computations are observational: no shadow token, confidence, or refresh decision
is written back, and random state is restored. All 427 main outputs, update counts, and
main model-call counts reproduce the normal $K=11$ reference. Shadow model calls are
counted separately and are excluded from efficiency measurements.

\paragraph{Outcome definitions and distance groups.}
A decision disagreement occurs if the next response, the set of newly accepted positions,
or the set of re-masked positions differs between the two alternatives. Different logits
alone do not count. Next-state Hamming distance counts token positions differing between
the resulting responses and is reported as a raw count, without division by response
length. Neither measure directly evaluates semantic correctness.

For the distance-bin intervals in Figure~\ref{fig:cache-validity-statistics}(a),
we resample all 427 problems with replacement, retain all observations belonging
to each sampled problem, and recompute pooled bin metrics (2,000 resamples, seed 42).
The distance, age, and controller-decision groups are given in
Table~\ref{tab:cache-groups}.

\paragraph{Equal-budget signal comparison.}
We exclude the 39 initial zero-age observations, leaving 5,291 candidate states from all
427 problems. Within each problem, distance and age each flag their highest-scoring
$20\%$ of candidate states. At a tied boundary, equal fractional weights are assigned
across tied states to reach exactly $0.2$ times the candidate count; selection never uses
the disagreement outcome. Residual disagreement and Hamming distance are weighted means
over the remaining $80\%$. We first compute each residual metric per problem and then
average over problems with equal weight, rather than pooling their unequal numbers of
observations.

Paired whole-problem bootstrap intervals for the equal-budget comparison use
2,000 resamples and seed 42, preserving both scores on each sampled problem.
For descriptive partial Spearman correlations, we rank each signal and outcome,
linearly residualize both on the other ranked signal, and correlate the residuals.
Appendix~\ref{app:cache-validity} reports grouped, equal-budget, and correlation results.

\subsubsection{Independent Generation Timing}
\label{app:mechanism-timing}

Figure~\ref{fig:feedback-statistics}(b) uses three complete 427-problem MBPP runs for each
of Saber and Archer, with the shared model and decoding configuration above and $K=11$
for Archer. The runs use NVIDIA A100-PCIE-40GB GPUs, PyTorch 2.1.2, and CUDA 12.1. Generation
time is measured with CUDA synchronization and includes cache builds, refreshes, cached
forwards, and decoder control; loading, tokenization, evaluation, interventions, and shadow
forwards are excluded. Each repeat first averages time over all problems; the plotted bars
are the mean of these three repeat means, with their sample standard deviation as error bars.

The repeat-level results and model-call breakdown appear in Appendix~\ref{app:feedback}.

\subsection{Scaling-Study Protocols}
\label{app:scaling-protocol}

\paragraph{Prompt length.}
We hold the response length at $G=256$ and increase the fixed prompt from 128
to 2,048 tokens. At each length, full and cached updates start from the same
state and are measured after warm-up, with CUDA synchronization immediately
before and after every timed call.

\paragraph{Response length.}
A complementary sweep fixes $P=512$ and $K=8$ and evaluates response budgets
$G\in\{256,512,1024\}$. Unlike the per-update study, this comparison measures
the complete realized trajectory, including cache construction and refreshes.
Table~\ref{tab:app-response-scaling} reports the generation times and their
Saber/Archer ratios.

\end{document}